\documentclass[nohyperref]{article}

\usepackage{microtype}
\usepackage{graphicx}
\usepackage{subfigure}
\usepackage{cuted}
\usepackage{booktabs} 

\newcommand{\norm}[1]{\left|\left| #1 \right|\right|}

\newcommand{\set}[1]{\left\{ #1 \right\}}

\usepackage{hyperref}

\usepackage[accepted]{icml2023}

\usepackage{amsmath}
\usepackage{amssymb}
\usepackage{mathtools}
\usepackage{amsthm}

\usepackage[capitalize,noabbrev]{cleveref}

\theoremstyle{plain}
\newtheorem{theorem}{Theorem}[section]

\newtheorem{corollary}[theorem]{Corollary}
\theoremstyle{definition}

\theoremstyle{remark}

\input{my_symbol.sty}

\usepackage[textsize=tiny]{todonotes}

\icmltitlerunning{Graph Neural Tangent Kernel: Convergence on Large Graphs}

\begin{document}

\twocolumn[
\icmltitle{Graph Neural Tangent Kernel: Convergence on Large Graphs}



\icmlsetsymbol{equal}{*}

\begin{icmlauthorlist}
\icmlauthor{Sanjukta Krishnagopal}{equal,yyy,xxx,sch}
\icmlauthor{Luana Ruiz}{equal,sch,comp}
\end{icmlauthorlist}

\icmlaffiliation{yyy}{Dept. of Electrical Engineering and Computer Science}
\icmlaffiliation{xxx}{Dept. of Mathematics, UCLA}
\icmlaffiliation{comp}{MIT CSAIL}
\icmlaffiliation{sch}{Work done in part while visiting the Simons Institute for the Theory of Computing.}

\icmlcorrespondingauthor{Sanjukta Krishnagopal}{sanjukta@berkeley.edu}
\icmlcorrespondingauthor{Luana Ruiz}{ruizl@mit.edu}

\icmlkeywords{keywords}

\vskip 0.3in
]


\newcommand{\sanjukta}[1]{\textcolor{blue}{#1} }


\printAffiliationsAndNotice{\icmlEqualContribution} 

\begin{abstract}
Graph neural networks (GNNs) achieve remarkable performance in graph machine learning tasks but can be hard to train on large-graph data, where their learning dynamics are not well understood. We investigate the training dynamics of large-graph GNNs using graph neural tangent kernels (GNTKs) and graphons. In the limit of large width, optimization of an overparametrized NN is equivalent to kernel regression on the NTK. Here, we investigate how the GNTK evolves as another independent dimension is varied: the graph size. We use graphons to define limit objects---graphon NNs for GNNs, and graphon NTKs for GNTKs---, and prove that, on a sequence of graphs, the GNTKs converge to the graphon NTK. We further prove that the spectrum of the GNTK, which is related to the directions of fastest learning which becomes relevant during early stopping, converges to the spectrum of the graphon NTK. This implies that in the large-graph limit, the GNTK fitted on a graph of moderate size can be used to solve the same task on the large graph, and to infer the learning dynamics of the large-graph GNN. These results are verified empirically on node regression and classification tasks.
\end{abstract}

\section{Introduction}
\label{sec:intro}

 Several real-world systems such as social-interactions, brain-connectome, epidemic spread, recommender systems, and traffic patterns are best represented by structured data in the form of large graphs. Graph neural networks (GNNs) are deep neural network architectures that leverage these graph structures to learn meaningful representations of node and edge data \cite{kipf17-classifgcnn,hamilton2017inductive,defferrard17-cnngraphs,gama18-gnnarchit}. GNNs have shown remarkable empirical performance in a number of graph machine learning tasks, but can be hard to train on large-graph data. Recent research efforts have attempted to understand the large-graph behavior of GNNs, and in particular why GNNs trained on small graphs scale well to large networks \cite{ruiz20-transf,levie2019transferability,keriven2020convergence}. However, the specific learning dynamics of a GNN trained directly on the large network, which are known to be challenging, are not as well understood.

Modern deep neural networks (DNNs) are typically overparametrized. The benefits of overparametrization include faster convergence \cite{allen2019convergence, arora2018optimization} and better generalization \cite{cao2019generalization}, but on the other hand
 the parameters can be more difficult to interpret and the learning dynamics harder to understand. A remarkable contribution of \citet{jacot2018neural} was the observation that, in the infinite width limit, learning the weights of a DNN via gradient descent reduces to kernel regression with a deterministic and fixed kernel called the \emph{neural tangent kernel} (NTK), which captures the first-order approximation of the neural network's evolution during gradient descent \cite{lee2019wide}. 
Since its introduction, the NTK has been an important and widely-studied tool in the machine learning toolbox, and kernel regression using the NTK has shown strong performance on small datasets \cite{arora2019harnessing}. 

In the graph case, it is straightforward to define the NTK associated with a GNN, or the graph neural tangent kernel (GNTK) \cite{NEURIPS2019_663fd3c5}. The GNTK allows studying the training dynamics of the GNN when the number of features (the analog of width in the DNN) is large.~Nonetheless, the effect of \textit{graph size} on the GNN learning dynamics, which can itself be thought of as another `width dimension', is not well understood, and there is limited research that investigates this rigorously. Herein, we propose to to understand the training dynamics of GNNs that are wide in both of these senses by combining the NTK formalism with the theory of graphons \cite{lovasz2012large,borgs2008convergent}.

A graphon is a symmetric bounded measurable function~$\bbW: [0,1]^2 \rightarrow [0,1]$ representing the limit of a sequence of dense graphs. It can also be interpreted as a random graph model, in which case we can use the graphon to sample stochastic graphs. The interpretation of $\bbW$ as both a graph limit and a random graph model makes it so that each graphon defines a family of similar graphs. Hence, one can expect properties of a graphon to generalize, in a probabilistic sense, the properties of graphs belonging to its family. Graphons have been used to study the limit behavior of GNNs, which converge to so-called graphon neural networks (WNNs) \cite{ruiz20-transf}. The fact that GNNs have a limit on the graphon implies that they are transferable across graphs in the same family, thus allowing a GNN to be trained on a graph of moderate size and transferred to a larger graph.


\subsection{Contributions}

In this paper, our first contribution is to define the graphon NTK (WNTK) associated with the WNN (Sec. \ref{sec:ntk}).
We then prove, using mathematical induction, that the GNTK converges to the WNTK (Thm. \ref{theorem1}). In practice, this implies that GNTKs, like GNNs, are transferable across graphs of different sizes associated with the same graphon. That is to say, one can subsample a small graph and the corresponding data from a large graph, then fit the subsampled data to the small-graph GNTK via kernel regression, and then transfer the fitted model to the large-graph GNTK, which is particularly important for large graphs.

A more important implication of the convergence of the GNTK is that it is possible to understand the training dynamics of GNNs on large graphs by analyzing the behavior of the corresponding GNTK in graphs of moderate size. 
For instance, the eigenvalues of the NTK are associated with the speed of convergence along the corresponding eigendirections \cite{jacot2018neural}. In Thm. \ref{theorem2}, we show that the eigenvalues of the graph GNTK, which indicate the directions of fastest convergence of the GNN, converge to the eigenvalues of the WNTK. This allows these eigenvalues to be estimated from GNTKs associated with smaller graphs. Lastly, we verify our theoretical results in three numerical applications: prediction of opinion dynamics on random graphs, movie recommendation using the MovieLens dataset, and node classification on the Cora, CiteSeer and PubMed networks. We observe the convergence of the GNTK (Sec. \ref{sbs:sims_conv}), the effect of width in kernel regression and GNN training (Sec. \ref{sbs:wide_sims}), and the convergence of the GNTK eigenvalues on sequences of graphs (Sec. \ref{sbs:eig_sims}).

\comment{
 Let $\bbPhi(\bbx,\ccalH)$ be a neural network where $\bbx \in \reals^d$ is the input data and $\ccalH$ is a set containing the neural network's learnable weights, which we stack in the vector $\bbh \in \mbR^{|\ccalH|}$. Consider the first-order Taylor approximation of $\bbPhi$ with respect to $\bbh$, explicitly
\begin{equation}
\bbPhi(\bbx,\ccalH) - \bbPhi(\bbx,\ccalH_0) \approx \nabla_\bbh \bbPhi(\bbx,\ccalH_0)^T(\bbh-\bbh_0) \text{.}
\end{equation}
Since this expression is linear in the model weights $\bbh$, if the neural network is close to its linear approximation, finding the optimal weights $\bbh$ consists of solving a linear regression where the inputs $\bbx$ are modified by the function $\nabla_\bbh \bbPhi(\bbx,\ccalH_0)^T$. Seeing this function as a feature map, we can define the corresponding kernel $\bbTheta(\bbx,\bbx')=\nabla_\bbh\bbPhi(\bbx,\ccalH_0)^T\nabla_\bbh\bbPhi(\bbx',\ccalH_0)$. This kernel---which depends only on the initial weights $\ccalH_0$---is called the neural tangent kernel or NTK \cite{jacot2018neural}.
In \cite{jacot2018neural}, it was shown that, in the infinite width limit as $d \to \infty$, the neural network $\bbPhi$ behaves very similarly to its linear approximation. To see this, let $\bby \in \reals^d$ be target output we are trying to approximate with $\bbPhi$. When the neural network has infinitely many parameters, a small change in $\ccalH$ can lead to a large change in $\bbPhi(\bbx,\ccalH)$. Hence, the change in $\ccalH$ has to be very small in order for $\bbPhi$ not to deviate from its current estimate of $\bby$. In other words, $\ccalH$ stays almost constant (close to $\ccalH_0$), making it so that the first-order Taylor approximation is accurate. In practice, the NTK is important because it helps understand the training dynamics of wide neural networks. For example, the eigenvalues of $\bbTheta$ determine the speed at which training converges along each eigendirection of the kernel. The NTK associated with standard, fully connected neural networks has been the subject of much research, and a handful of works have studied the graph NTK (GNTK) associated with GNNs where the graph $\bbG$ is seen as the input data, such as in graph classification, where the graphs (e.g., molecules) tend to be small. On the other hand, the behavior of wide GNNs on node-level tasks, where the node data varies but the graphs are fixed and  typically large (e.g., user similarity networks in recommendation systems), remains to be understood. At Simons, \textbf{Sanjukta Krishnagopal} and I worked on this problem by defining the appropriate GNTK and studying its convergence to the graphon NTK on dense graphs. The convergence of the GNTK follows from the fact that the graph spectra converges to the graphon spectra, and uses similar steps as the proof of convergence of GNNs in \cite{jacot2018neural}. In a preliminary set of experiments, we observe that, in the same way that the NTK provides a good approximation of fully connected neural networks as their width increases, the GNTK provides a good approximation of the GNN as the number of features in its layers increases. This is illustrated in Fig. \ref{fig:wntk} for GNNs with $50$, $500$, and $1000$ features per layer trained to solve a consensus task. In our future work, we plan on running experiments on graphs of growing size, to analyze the convergence of the GNTK; and on understanding the influence of the graph and graphon spectra in the spectra of the GNTK and of the WNTK. This work will be featured in a paper under preparation for submission at ICML 2023.
}


\section{Related Work}
\label{sec:related}

\noindent \textbf{Neural tangent kernels.} The connection between infinitely wide DNNs and kernel methods (Gaussian processes) has been known since the 1990s \cite{neal1996priors, williams1996computing}, but a more theoretical formulation was presented by \citet{jacot2018neural}, which introduced the NTK and proved its constancy property in the infinite width limit, with \citet{liu2020linearity} later showing that constancy only holds for architectures with linear output layer. Several works have derived NTKs for a generalized classes of neural networks, including convolutional neural networks \cite{arora2019exact,li2019enhanced}, ResNets \cite{huang2020deep} and, most closely related to our work, GNNs \cite{NEURIPS2019_663fd3c5}.

\noindent \textbf{Graphons and size generalization in deep learning.} Graphons have been used to understand GNN convergence and transferability \cite{ruiz20-transf,ruiz2020graphonsp,maskey2023transferability}, to analyze the generalization properties of GNNs in large graphs \cite{maskey2022generalization}, and to propose more computationally efficient training algorithms for large-scale GNNs \cite{cervino2022training}. More recently, \citet{xia2022implicit} proposed implicit graphon neural representations, which use neural networks to estimate graphons. 

\citet{yehudai2021local} study graphs in which the local structure depends on the graph size---which is analogous to the dense graphs associated with graphons---, and find that GNNs are not guaranteed to scale to large graph sizes. However, their result foregoes the normalization of the adjacency matrix by the graph size, while we leverage this normalization to show operator norm convergence of the kernels associated with the graph convolution.
It is also worth noting that our work is fundamentally different than GNN convergence or transferability results, because it relates not to the GNN architecture but to the GNN \textit{learning dynamics}, and uses a mathematical induction proof as opposed to the spectral convergence argument typically used in GNN convergence proofs.


\section{Graph and Graphon Neural Networks}
\label{sec:prelim}

Let $\bbG_n=(\ccalV,\ccalE,w)$ be a graph where $\ccalV$, $|\ccalV|=n$, is the set of nodes or vertices, $\ccalE \subseteq \ccalV \times \ccalV$ is the set of edges and $w:\ccalE\to\reals$ is a map assigning weights to the edges in $\ccalE$. The \textit{size-normalized} adjacency matrix of $\bbG_n$, denoted $\bbA_n \in \reals^{n\times n}$, is given by $[\bbA]_{ij}=w(i,j)/n$. In this paper we focus on undirected graphs $\bbG_n$, with symmetric $\bbA_n$.

\subsection{Graph Neural Networks}
Supported on the graph $\bbG_n$ with $n$ nodes, we define node data $\bbx_n \in \reals^n$---also called \textit{graph signals} \cite{shuman13-mag}---where $[\bbx_n]_i$ is the value of the signal at node $i$. 
GNNs iteratively update each node's data by aggregating the data from its neighbors using \textit{graph convolutions}.
An order-$K$ graph convolution is defined as
\begin{equation} \label{eqn:graph_convolution}
\bby_n = H(\bbx_n) = \sum_{k=0}^{K-1} h_k \bbA_n^k \bbx_n
\end{equation}
where $h_0,\dots,h_{K-1}$ are the convolution coefficients. When $K=2$ and $\bbA_n$ is binary, \eqref{eqn:graph_convolution} can be seen as an aggregation operation akin to the $\texttt{AGGREGATE}$ operation in, e.g., \cite{xu2018how,hamilton2017inductive}, or as the message-passing operation in message-passing neural networks (MPNNs) \cite{gilmer2017neural}.

More generally, let $\bbX_n \in \reals^{n \times F}$ and $\bbY_n \in \reals^{n \times G}$ be $F$- and $G$-dimensional signals respectively, where the $f$th ($g$th) column is a \textit{feature} $\bbx_n^f$ ($\bby_n^g$). 
In this case, the graph convolution generalizes to
\begin{equation} \label{eqn:gen_graph_convolution}
\bbY_n = H(\bbX_n) = \sum_{k=0}^{K-1} \bbA_n^k \bbX_n \bbH_k
\end{equation}
with weights $\bbH_0, \bbH_1, \ldots, \bbH_{K-1} \in \reals^{F \times G}$. Note that the number of parameters of the graph convolutions in \eqref{eqn:graph_convolution} and \eqref{eqn:gen_graph_convolution}, $K$ and $KFG$ respectively, are independent of $n$.


A GNN consists of $L$ layers, each of which composes a graph convolution and a nonlinear activation function. Explicitly, the $l${th} layer of a GNN can be written as
\begin{align} \label{eqn:gcn_layer}
\begin{split}
\bbX_{l,n} &= \sigma \left(\bbU_{l,n}\right) \\
\bbU_{l,n} &= H(\bbX_{l-1,n}) = \sum_{k=0}^{K-1} \bbA_n^k \bbX_{l-1,n} \bbH_{l,k} 
\end{split}
\end{align}
where $\bbX_{l-1,n} \in \reals^{n \times F_{l-1}}$ is the layer input, $\bbX_{l,n} \in \reals^{n \times F_l}$ is the layer output, $\bbH_{l,k} \in \reals^{F_{l-1} \times F_l}$ are the layer weights and $\sigma$ is a pointwise nonlinear activation function, i.e., $[\sigma(\bbX)]_{ij}=\sigma([\bbX]_{ij})$. Typical choices for $\sigma$ include ReLU, tanh, and sigmoid. 
At the first layer of the GNN, the input $\bbX_{0,n}$ is the input data $\bbX$. Similarly, the GNN output $\bbY_n$ is given by the last layer output $\bbX_{n,L}$. 

In the following, we represent the entire GNN consisting of the concatenation of $L$ layers like \eqref{eqn:gcn_layer} as the parametric map $\bbY_n = f(\bbX_n; \bbA_n, \ccalH)$, where $\ccalH = \{\bbH_{l,k}\}_{l,k}$ groups the learnable weights $\bbH_{l,k}$ at all layers. This more concise representation highlights the independence between the adjacency matrix $\bbA_n$ (i.e., the graph) and the parameters $\ccalH$, which the GNN inherits from the graph convolution. We note that while we describe a general class of GNNs here, the methods in our paper can be directly adapted to other types of GNNs.

\subsection{Graphon Neural Networks}

\begin{table}[t]
    \caption{Description of variables.}
    \centering
    \begin{tabular}{c|c}
    \hline
        Variable & Description  \\ \hline
        $\bbA_n$ & Adjacency matrix of graph $\bbG_n$ \\
        $\bbx_n$ & Graph signal on $\bbG_n$\\
        $\bbW_n$ & Graphon induced by $\bbG_n$ \\ 
        $X_n$ & Graphon signal induced by $\bbx_n$\\
        $\bbW$ &  Limiting graphon \\
        $X$  & Limiting graphon signal \\ \hline
    \end{tabular}
    \label{tb:symbols}
\end{table}

Graphons are bounded, symmetric, measurable functions~$\bbW:[0,1]^2\to [0,1]$ representing limits of sequences of dense graphs \cite{lovasz2012large,borgs2008convergent}. 
A graph sequence $\{\bbG_n\}$ converges to a graphon in the sense that the densities of homomorphisms of any finite, unweighted and undirected graph $\bbF=(\ccalV',\ccalE')$ into $\bbG_n$ converge to the densities of homomorphisms of $\bbF$ into $\bbW$. The graphs $\bbF$ can be thought of as motifs, such as triangles, $k$-cycles, $k$-cliques, etc. Explicitly, 
let $\mbox{hom}(\bbF,\bbG_n)$ denote the total number of homomorphisms between $\ccalV'$ and $\ccalV_n$. 
The density of such homomorphisms is given by $t(\bbF,\bbG_n) = {\mbox{hom}(\bbF,\bbG_n)}/{n^{|\ccalV'|}}$ and we can similarly define $t(\bbF,\bbW$) (see \cite{borgs2008convergent}). Then, we say that $\bbG_n \to \bbW$ if and only if
\begin{equation}\label{eqn:graph_convergence}
   \lim_{n \to \infty} t(\bbF,\bbG_n) = t(\bbF, \bbW)
\end{equation}
for all simple $\bbF$ \cite{lovasz2006limits}. 

Alternatively, the graphon can also be seen as a generative model for stochastic (also called $\bbW$-random) graphs. Nodes are picked by sampling points $u_i$, $1 \leq i \leq n$, from the unit interval and connecting edges between nodes $i$ and $j$ with probability $\bbW(u_i,u_j)$. Importantly, sequences of stochastic graphs generated in this way converge almost surely to the graphon \cite{lovasz2012large}[Cor. 10.4].

The notion of graph signals is extended to graphons by defining graphon signals, which are functions $X:[0,1] \to\reals$ \cite{ruiz2020graphonsp}. We restrict attention to graphon signals with finite energy, i.e., $X \in L^2([0,1])$. 

Analogously to \eqref{eqn:gen_graph_convolution}, given $F$- and $G$-dimensional graphon signals $X:[0,1]\to\reals^F$ and $Y:[0,1]\to\reals^G$, the graphon convolution is defined as \cite{ruiz2020graphonsp}
\begin{align} \label{eqn:gen_graphon_convolution}
\begin{split}
Y = T_H X &= \sum_{k=0}^{K-1} T_{W}^{(k)} X \bbH_k \\
T_{W}^{(k)}X &= \int_0^1 \bbW(u,v)T_W^{(k-1)} X(u)du
\end{split}
\end{align}
where $T_{W}^{(0)}=\bbI$ is the identity and the convolution weights are collected in the matrices $\bbH_0, \bbH_1, \ldots, \bbH_{K-1} \in \reals^{F \times G}$.

The extension of the GNN to graphon data is the graphon neural network (WNN). Akin to the GNN, the WNN is formed by $L$ layers each of which composes a graphon convolution and a nonlinear activation function. Explicitly, the $l$th layer of the WNN is given by
\begin{align}\label{eqn:wcn_layer}
\begin{split}
X_{l} &= \sigma\left(U_l\right) \\
U_l &= T_{H_l} X_{l-1} = \sum_{k=0}^{K-1} T_{W}^{(k)} X_{l-1}\bbH_{l,k}
\end{split}
\end{align}
where $X_{l-1}: [0,1] \to \reals^{F_{l-1}}$ is the layer input,  $X_{l}: [0,1] \to \reals^{F_{l}}$ is its output, $\bbH_{l,k} \in \reals^{F_{l-1}\times F_l}$ are its weights and $\sigma$ is a pointwise nonlinearity (e.g., the ReLU).
The input of the first layer of the WNN is $X_{0}=X$, and the output of the WNN is the last layer ouput, i.e., $Y=X_L$.

We can describe the WNN more compactly as the map $Y = f(X;\bbW,\ccalH)$, with $\ccalH = \{\bbH_{l,k}\}_{l,k}$ the set of learnable parameters at all layers. Note that, if the weights $\ccalH$ are the same, the WNN map $f(X;\bbW,\ccalH)$ is the same as the GNN map $f(X;\bbA_n,\ccalH)$ with $\bbA_n$ swapped with $\bbW$. This is important because it implies that, similarly to how graphons are generative models for graphs, \emph{WNNs are generative models for GNNs}. Indeed, we can use the WNN $f(X;\bbW,\ccalH)$ to sample the GNN $\bby_n = f(\bbx_n;\bbA_n,\ccalH)$ with
\begin{align} \label{eqn:gcn_obtained}
&[\bbA_n]_{ij} \sim \mbox{Ber}(\bbW(u_i,u_j))\nonumber \\
&[\bbx_n]_i = X(u_i) 
\end{align}
where the $u_i$ are sampled uniformly and independently at random from $[0,1]$ and $\mbox{Ber}()$ is the Bernoulli distribution.

As $n \to \infty$, sequences of GNNs sampled from the WNN as described above converge to the WNN. In fact, as long as $\ccalH$ is the same, any GNN $f(\bbx_n;\bbA_n,\ccalH)$ applied to a sequence $\{(\bbG_n,\bbx_n)\}$ converging to $(\bbW,X)$ converges to $f(X;\bbW,\ccalH)$ \cite{ruiz2020graphonsp}. A more important result in practice is that this convergence implies that GNNs are transferable across graphs associated with the same graphon, i.e., they can be trained on graphs $\bbG_n$ and  executed on graphs $\bbG_m$ with an error that decreases asymptotically with $n$ and $m$ \cite{ruiz2021transferability,ruiz20-transf}. 

In Sec. \ref{sec:convergence}, we prove a similar convergence result for the neural tangent kernels (NTKs) associated with $f(\bbx_n;\bbA_n,\ccalH)$ and $f(X;\bbW,\ccalH)$, but before doing so, we need to introduce \textit{induced WNNs}.
The WNN induced by the GNN $f(\bbx_n;\bbA_n,\ccalH)$ is defined as $Y_n = f(X_n;\bbW_n,\ccalH)$, with
\begin{align} \label{eqn:wnn_induced}
\begin{split}
&\bbW_n(u,v) = \sum_{i=1}^n\sum_{j=1}^n [\bbA_n]_{ij} \mbI(u \in I_i) \mbI(v \in I_j), \\
&X_n(u) = \sum_{i=1}^n [\bbx_n]_i \mbI(u \in I_i)
\end{split}
\end{align}
and where $\mbI$ is the indicator function, $I_i=[(i-1)/n,i/n)$ for $1 \leq i \leq n-1$, and $I_n = [(n-1)/n,1]$. The graphon $\bbW_n$ is \textit{induced by the graph} $\bbG_n$, and the graphon signals $X_n$ and $Y_n$ are \textit{ induced by the graph signals} $\bbx_n$ and $\bby_n$. A succinct description of all graph and graphon variables is provided in Table \ref{tb:symbols}.

\section{Graph and Graphon Neural Tangent Kernel} \label{sec:ntk}

Consider a general, fully-connected neural network $f(\bbx;\ccalH)$, with layers given by $\bbx_l = \sigma(\bbH_l\bbx_{l-1})$, input $\bbx_0 = \bbx \in \reals^{d_0}$ and learnable parameters $\ccalH=\{\bbH_l\}_l \in \reals^{d_{l-1} \times d_l}$, $[\bbH_l]_{pq} = h_{l,pq}$. For a training set $\{\bbx_i, \tby_i\}_{i=1}^M$, assume that the loss to be minimized is the mean squared error (MSE) or quadratic loss
\begin{equation}
    \min_\ccalH \ell(\ccalH) = \min_\ccalH \sum_{i=1}^M (f(\bbx_i;\ccalH) - \tby_i)^2.
\end{equation} 
As the training progresses, the output of the neural network for input $\bbx_i$ is updated as $\bby_i(t) = f(\bbx_i;\ccalH(t))$. As such, the weight update rule is given by
$\smash{{\frac{\partial \ccalH}{\partial t} = - \nabla \ell(\ccalH(t))}}$
and so the output evolves as
\begin{align}
\begin{split}
    &\frac{\partial \bby_i (t)}{\partial t} = \textstyle\sum_j \Theta(\bbx_i,\bbx_j;\ccalH(t)) (f(\bbx_i;\ccalH(t)) - \hby_i), \\ 
    &\Theta(\bbx_i,\bbx_j; \ccalH)\ = \sum_{l,p,q} \frac{\partial f(\bbx_i;\ccalH)}{\partial h_{l,pq}} \frac{\partial f(\bbx_j;\ccalH)}{\partial h_{l,pq}}. 
\label{eq:NTK}
\end{split}
\end{align}


In the infinite-width limit as $d_l \to \infty$, \citet{jacot2018neural,liu2020linearity} showed that, provided that the last layer has linear output (i.e., it does not have an activation $\sigma$), $\Theta(\bbx_i,\bbx_j;\ccalH)$ converges to a limiting kernel, the NTK. This kernel stays constant during training---a property called constancy---, which is equivalent to replacing the outputs of the neural network by their first-order Taylor expansion in the parameter space \cite{lee2019wide}. Hence, in the infinite-width limit the training dynamics of \eqref{eq:NTK} reduce to kernel ridge regression\footnote{In practice, it is not necessary to consider the MSE for this derivation. It suffices for the loss to be such that the norm of the training direction $f(\ccalH_0)-f(\ccalH^\star)$ is strictly decreasing during training \cite{jacot2018neural}. E.g., if we consider the cross-entropy loss, the training dynamics reduce to kernel logistic regression.} on the NTK, which has a closed-form solution. This facilitates understanding the learning dynamics of overparametrized neural networks, which are notoriously difficult to study directly, by analysis of the corresponding NTK. 

\subsection{Graph Neural Tangent Kernel}

For simplicity, we will consider a GNN with only one feature per layer; the generalization to multiple features is more involved but straightforward. 
Recall that the $L$-layer GNN supported on $\bbA_n \in \reals^{n\times n}$ is written as
\begin{align*}
  &\bbx_{n,0} = \bbx_n \quad \quad \quad \ \ \ \bbu_{l,n} = H_l (\bbx_{l - 1,n}) \\
  &\bbx_{l,n} = \sigma(\bbu_{l,n}) \quad \quad f(\bbx_n; \bbA_n, \ccalH) = \bbx_{L,n} \text{.} 
\end{align*}
{
\setlength{\jot}{11pt}
The graph NTK (GNTK) associated with this GNN is 
\begin{align*}
  \Theta (\bbx_n,\bbx'_n;\bbA_n,\ccalH) = &\nabla_\ccalH f(\bbx_n;\bbA_n,\ccalH)^T \nabla_\ccalH f(\bbx'_n;\bbA_n,\ccalH) \\
  & \hspace{-2cm} = \textstyle\sum_{k} \sigma'(\bbu_{L,n}) \bbA_n^{k} \bbx_{L-1,n} \otimes \sigma'(\bbu_{L,n}') \bbA_n^{k} \bbx_{L-1,n}' \\
  &\quad \hspace{-2cm} + \sigma'(\bbu_{L,n} )   (\sigma'(\bbu_{L-1,n}) \bbA_n^{k} \bbx_{L-2,n}) \\
  &\quad \quad  \hspace{-2cm} \otimes  \sigma'(\bbu_{L,n}') H_L(\sigma'(\bbu_{L-1,n}') \bbA_n^{k} \bbx'_{L-2,n}) + \ldots
\end{align*}
where there are $L$ terms; the first term corresponds to the last layer, the second term to the second-last layer, and so on. 
We have used $\sigma'(\bbu)$ to denote the diagonal matrix with $jj$th entry equal to $\sigma'([\bbu]_j)$, where $\sigma'$ is the derivative of $\sigma$. Note that $\nabla_\ccalH f$ is a $|\ccalH| \times n$ matrix, hence the GNTK is a matrix $\Theta_n( \bbx_n, \bbx_n';\bbA_n,\ccalH) \in \reals^{n \times n}$. 
}
\subsection{Graphon Neural Tangent Kernel}

To derive the WNTK, we will also consider a WNN with only one feature per layer for simplicity. Recall that the $L$-layer WNN associated with the graphon $\bbW$ is given by
\begin{align*}
&X_0 = X \quad  \quad \quad \ \  U_l = \textstyle T_{H_l}X_{l - 1}\\
&X_l = \sigma(U_l) \quad  \quad f(X;\bbW,\ccalH) = X_L \text{.}
\end{align*}
{
\setlength{\jot}{11pt}
Therefore, the WNTK is given by
\begin{align} \label{eqn:wntk}
\begin{split} 
  \Theta(X,X';\bbW,\ccalH) &= \\ 
  & \hspace{-2.5cm} = \textstyle\sum_{l,k} \partial_{h_{l,k}} f(X;\bbW,\ccalH) \otimes \partial_{h_{l,k}} f(X';\bbW,\ccalH) \text{.}
\end{split}
\end{align}

Calculating the derivative with respect to the $k$th weight in layer $l_j=L-j$, we get
\begin{equation*}
    \partial_{h_{l_j,k}} u_L = 
    \textstyle T_{W}^{(k)} X_{L-1}
\end{equation*}
for $j=0$, and similarly for $0 < j < L$,
\begin{equation*}
    \partial_{h_{l_j,k}} u_L = \textstyle T_{H_L} \sigma^\prime(U_{L-1}) \dots \textstyle T_{H_{l{j-1}}} \sigma'(U_{L-j}) \textstyle T_W^{(k)} X_{l_{j-1}} \text{.}
\end{equation*}
Hence, \eqref{eqn:wntk} has $L$ terms in total, explicitly
\begin{align*}
\begin{split}
  \Theta(X,X';\bbW,\ccalH) &= \\ 
  & \hspace{-1.8cm} = \textstyle\sum_{k}  \sigma'(U_L) T_{W}^{(k)} X_{L-1} \otimes \sigma'(U_L') T_{W}^{(k)} X_{L-1}' \\
  &\quad \hspace{-1.8cm} + \sigma'(U_L) T_{H_L}(\sigma'(U_{L-1}) T_W^{(k)} X_{L-2}) \\
  &\quad \hspace{-1.8cm}  \quad \otimes \sigma'(U_L') T_{H_L} (\sigma'(U'_{L-1}) T_W^{(k)} X'_{L-2}) + \ldots 
\end{split}
\end{align*}
where $\sigma'(U_L)$ is now evaluated and multiplied pointwise. 
Note that $\Theta$ is a linear operator on graphon signals, i.e., given a signal $Y \in L^2([0,1])$, the WNTK applied to this signal yields 
\begin{align*}
  \Theta(X,X';\bbW,\ccalH)Y &= \textstyle\int_0^1 \textstyle\sum_{k} \sigma'(U_L(u)) T_{W}^{(k)} X_{L-1}(u)\\
  & \quad \, \, \times \sigma'(U_L'(v))  T_{W}^{(k)} X_{L-1}'(v) Y(v) \\ 
  & \hspace{-3.1cm} \quad + \sigma'(U_L(u)) T_{H_L}(\sigma'(U_{L-1}) T_W^{(k)} X_{L-2})(u) \\ 
  & \quad \, \hspace{-3.1cm} \, \times \sigma'(U_L(v)) T_{H_L}(\sigma'(U_{L-1}) T_W^{(k)} X_{L-2})(v)Y(v) +\ldots \, dv
\end{align*}
where the first term corresponds to the last layer, the second corrresponds to the second-last layer and so on. 

For a graph $\bbA_n$ sampled from $\bbW$ with associated induced graphon $\bbW_n$, and for input features $\bbx_n$ sampled from a graphon signal $X$ and associated induced graphon signal $X_n$ [cf. \eqref{eqn:wnn_induced}],
the induced WNTK is given by the same formula as \eqref{eqn:wntk} but with $\bbW$, $X$ replaced by $\bbW_n$, $X_n$.
}
\section{Graph NTK Converges to Graphon NTK}
\label{sec:convergence}

Next, we consider a sequence of graph signals $\{(\bbG_n,\bbx_n)\}$ converging to a graphon signal $(\bbW,X)$ to evaluate the convergence of the GNTK $\Theta(\bbx_n,\bbx_n';\bbA_n,\ccalH)$ to the WNTK $\Theta(X,X';\bbW,\ccalH)$. 
First, for each pair of signals $X,X'$, the WNTK $\Theta(X,X';\bbW,\ccalH)$---which we denote $\Theta(X,X')$ to simplify notation---is a linear operator on functions in $L^2$. Hence, to prove convergence to this operator we introduce the operator norm 
\begin{align}
\norm{\Theta(X,X')} = \sup_{Y \neq 0} \frac{\norm{\Theta(X,X')Y}}{\norm{Y}} \text{.}
\label{eq:normdef}
\end{align}


\begin{figure*}[t]
\centering
\includegraphics[width=0.47\linewidth]{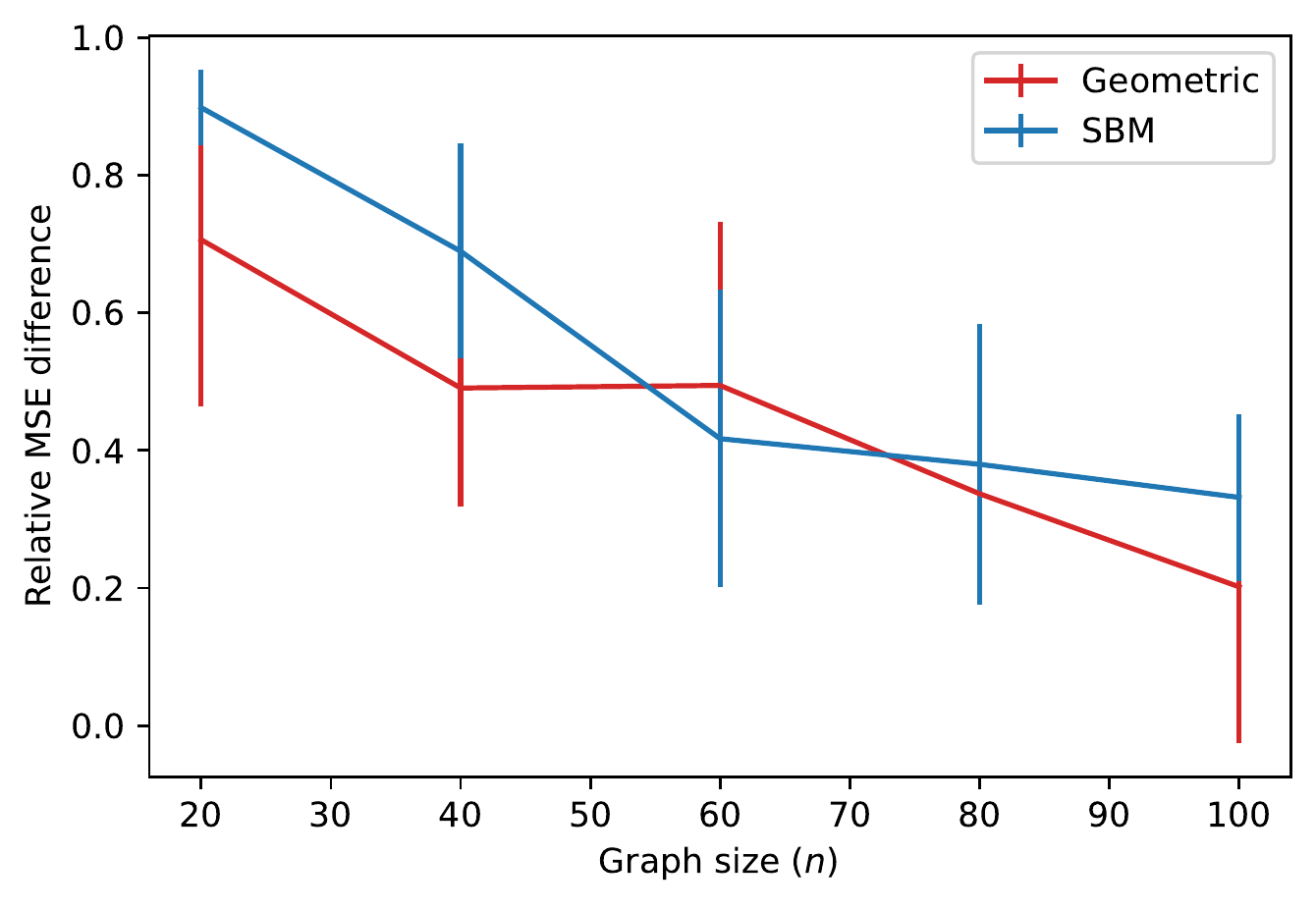}
\includegraphics[width=0.48\linewidth]{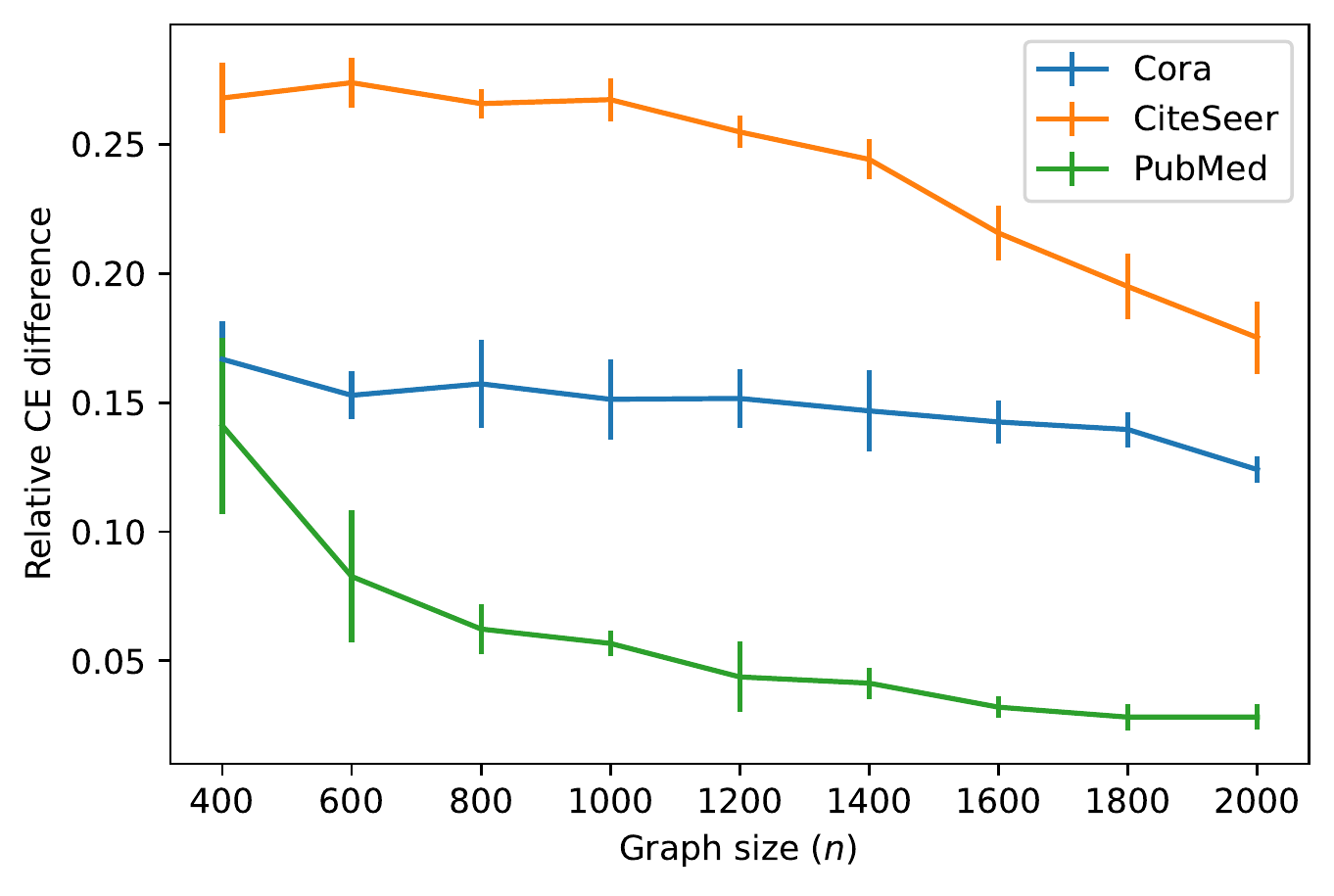}
\caption{Difference between the test error achieved by the GNTK fitted on the $n$-node graph when used for prediction the $n$-node graph, and the test error achived by the same GNTK when used for predicition on the $N$-node graph. (left) Opinion dynamics on geometric and SBM graphs, where the test error is the MSE $N=300$. (right) Node classification on Cora, CiteSeer and Pubmed, where the test error is the CE and $N=2708$, $3327$, and $10000$ respectively.}
\label{fig:convergence}
\end{figure*}

Let $\bbW$ be a graphon and $X,X'$ be fixed graphon signals.
Consider a sequence of graph signals $\{(\bbG_n,\bbx_n)\}$ (respectively $\{(\bbG_n,\bbx_n')\}$) converging to $(\bbW,X)$ (respectively $(\bbW,X')$) in the sense of \cite{ruiz2020graphonsp}[Def. 2]. The notion of convergence of graph signals on a sequence of dense graphs is simple; $\bbG_n$ converges to $\bbW$ in the homomorphism density sense [cf. \eqref{eqn:graph_convergence}], and $X_n$, the graphon signal induced by $\bbx_n$, converges to $X$ in the $L^2$ norm, i.e.
\begin{equation} \label{eqn:signal_convergence}
\|X_n-X\| \to 0
\end{equation} 
up to node relabelings (see \cite{ruiz2020graphonsp} for further details, and in particular Lemma 2 for the relationship between $(\bbG_n,\bbx_n)$ and the induced signal $(\bbW_n,X_n)$).

Let $(\bbW_n,X_n)$ be the graphon signal induced by the graph signal $(\bbG_n,\bbx_n)$ [cf. \eqref{eqn:wnn_induced}].
The main result of this paper is the following theorem that shows that the WNTK induced by the GNTK converges to the limiting WNTK. The proof is deferred to Appendix \ref{appendix:convergence_proof}.

\begin{theorem}
Let $\textbf{W}$ be a graphon and $X,X' \in L^2([0,1])$ be arbitrary graphon signals.  
Suppose that $\set{\bbG_n}$ is a sequence of graphs converging to $\bbW$ [cf. \eqref{eqn:graph_convergence}] and $\set{\bbx_n}, \set{\bbx_n'}$ are sequences of graph signals converging to $X$ [cf. \eqref{eqn:signal_convergence}].
\comment{
\begin{align*}
& \lim_{n \to \infty} ||X_n - X|| = 0 \\     
& \lim_{n \to \infty} ||X_n' - X'|| = 0 \\
& \lim_{n \to \infty} ||W_n - W|| = 0. 
\end{align*}
}
Then, for any $L$-layer GNN with finite $K$, fixed weights $\ccalH$ and $1$-Lipschitz nonlinearity $\sigma$, the associated GNTKs $\Theta(\bbx_n,\bbx'_n;\bbA_n,\ccalH)$ converge in the operator norm: 
\begin{align*}
\lim_{n \to \infty}\norm{\Theta(X_n,X'_n;\bbW_n,\ccalH)-\Theta(X,X';\bbW,\ccalH)} = 0
\end{align*}
where $\Theta(X_n,X'_n;\bbW_n,\ccalH)$ is the WNTK induced by the GNTK $\Theta(\bbx_n,\bbx_n';\bbA_n,\ccalH)$.
\label{theorem1}
\end{theorem}

The convergence of the GNTK to a limit object---the WNTK---has two important implications for machine learning on large-scale graphs. First, consider that we have a large graph $\bbG_N$ on which we want to predict signals $\bby_N$ using the GNTK, but that we do not have enough computational resources to compute the full-sized GNTK $\Theta(\bbx_N,\bbx_N';\bbA_N,\ccalH)$. This is expected, since calculating the kernel regression weights on this graph requires inverting a matrix with dimension proportional to $M$, the number of training samples, and the graph size $N$, see Sec. \ref{sec:ntk}. Thm. \ref{theorem1} implies that we can subsample the GNTK and the labels $Y_N$ on a smaller graph $\bbG_n$, $n \ll N$---as $\Theta(\bbx_n,\bbx_n';\bbA_n,\ccalH)$ and $\bby_n$ respectively---and fit the data to the GNTK on this smaller graph. Once the kernel regression weights are obtained, they can then be \textit{transferred} to make predictions on $\bbG_N$. Naturally, there will be an error associated with this transference; however, due to convergence, this error vanishes asymptotically in $n$. 

The precise convergence bounds for different graphs can be computed by plugging into \eqref{eq:nonasymp}. For instance, for graphs sampled from the graphon as in \cite{ruiz2021transferability}[Def. 1], the bound is given by
\begin{align*}
 \norm{\Theta(X_n,X_n') - \Theta(X,X')} \leq C (K^{4+L} \frac{2 A_w}{n} + K^{2+L}\frac{A_x}{n}).
\end{align*}
where $C$ is a constant, and $A_w$ and $A_x$ are the Lipschitz constants of the graphon and the graphon signal respectively.

In practice, transferability is useful in a variety of real-world applications, especially on large graphs. For instance, one may be interested in the learning dynamics of a GNN trained to provide ad recommendations on a growing, dense social media network. Our theoretical results show that one can calculate the GNTK on a smaller subgraph, e.g., the same social media network at an earlier time, and then transfer it to the large target graph with theoretical performance guarantees. This works because dense networks growing under the same underlying `rule' or `process' should have similar homomorphism densities, and converge to similar graphons; or, inversely, they can be thought of as being sampled from similar graphons. The more similar these limiting graphons are (under a metric such as the cut metric or the $L^2$ norm), the lower the error when transferring the GNTK across graphs.

A second implication of Thm. \ref{theorem1}, which is perhaps more important, is that GNTK convergence implies that a GNTK fitted on a smaller graph can be used to infer details about the \textit{training dynamics} of wide/overparametrized GNNs on larger graphs. This follows from Thm. \ref{theorem1} in conjunction with theoretical results by \citet{liu2020linearity}, which proves that in the infinite-width limit the training dynamics of general convolutional models (including GNNs) reduces to kernel ridge regression on the corresponding NTK; and empirical results by \citet{sabanayagam2021new} showing that even when the output layer is nonlinear, constancy of the GNTK can still be observed.

An example of property of a large-graph GNTK $\Theta(\bbx_N,\bbx_N'; \bbA_N, \ccalH)$ that can be estimated from a small-graph GNTK $\Theta(\bbx_n,\bbx_n'; \bbA_n, \ccalH)$ is its eigenvalue spectrum. As we show in Thm. \ref{theorem2}, the spectrum of $\Theta(X_n,X_n'; \bbW_n, \ccalH)$ converges to the the spectrum of $\Theta(X,X'; \bbW, \ccalH)$. The proof is deferred to Appendix \ref{appendix:eigenvalue_proof}.

\begin{theorem} \label{theorem2}
Let $\{\bbx_{i,n}\}_{i=1}^M$ be a set of $M$ signals on the graph $\bbG_n$, and $\{X_{i,n}\}_{i=1}^M$ the corresponding induced signals on the induced graphon $\bbW_n$. Assume that $\bbG_n \to \bbW$ [cf. \eqref{eqn:graph_convergence}] and $X_{i,n} \to X$ [cf. \eqref{eqn:signal_convergence}] for all $i$. Define the operators $\bbTheta_n $ and $\bbTheta$ where $[\bbTheta_n]_{ij} = \Theta(X_{i,n},X_{j,n};\bbW_n,\ccalH)$ and $[\bbTheta]_{ij} = \Theta(X_{i},X_{j};\bbW,\ccalH)$. Let $\lambda_p(T)$, $p \in \mbZ \setminus \{0\}$, denote the $p$th eigenvalue of a compact, self-adjoint operator $T$, with $\lambda_{-1} < \lambda_{-2} < \ldots \leq 0 \leq \ldots < \lambda_2 < \lambda_1$.
If the weights in $\ccalH$ are bounded, then, for all $p$,
\begin{equation*}
\lim_{n\to \infty} |\lambda_p(\bbTheta_n)-\lambda_p(\bbTheta)| \to 0 \text{.}
\end{equation*}
Moreover, the corresponding eigenspaces converge in the sense that the spectral projectors converge.
\end{theorem}

This is an important result because, as proved by \citet{jacot2018neural}, the convergence of kernel gradient descent follows the kernel principal components. Thus, if the GNN $f(\ccalH)$ is sufficiently wide, and if $f(\ccalH_0)-f(\ccalH^\star)$ is aligned with the $i$th GNTK principal component, we can expect gradient descent to converge with rate proportional to $\lambda_i$, the $i$th eigenvalue of the GNTK. Thm. \ref{theorem2} shows that the GNTK eigenvalues converge to the eigenvalues of the WNTK. As such, we can estimate the speed of convergence of a large-graph GNN $f(\bbx_N; \bbA_N, \ccalH)$ from the eigenvalues of a small-graph GNTK $\Theta(\bbx_n,\bbx_n'; \bbA_n, \ccalH)$. A numerical example of this application of Thms. \ref{theorem1} and \ref{theorem2} is given in Sec. \ref{sbs:eig_sims}.

Spectral convergence of the GNTK is particularly relevant for practical applications on large graphs. Convergence of the spectrum implies convergence of the directions of learning. Intuitively, one can calculate the GNTK on the smaller graph, then transfer it to the large graph with the guarantee that the dominant directions of fastest convergence are being captured. We can learn GNN architectures that harness this property by pretraining on a smaller graph, and making adjustments on the large graph; or, by iterating between different resolutions of graph size.

\comment{
As a final comment, we point out that one can also deduce convergence of the GNTK spectrum to the WNTK spectrum. \sanjukta{The eigenvalues of the NTKs capture the rate of convergence along the directions that are captured by the eigenvectors. For instance, the eigenvector corresponds to the leading eigenvalue is the direction of fastest convergence.}\red{terminology question: should we use eigenfunctions instead for the WNTK?}
\sanjukta{still working on this. Do you happen to know the citation below off the top of your head?}

\begin{theorem}
For a bounded, linear operator $Q:L^2 \to L^2$ denote $\sigma(Q)$ the spectrum of $Q$.
If $W$ is Lipschitz and $W_n \to W$ in the Lipschitz metric and the coefficients $h_{k,n}$ are bounded uniformly, then for all $X,X'$ and sequences such that $\norm{X_n - X} \to 0$ and $\norm{X_n' - X'} \to 0$,
\begin{align*}
\sup_{z\in \sigma(\Theta(X,X'))} \inf_{z' \in \sigma(\Theta(X_n,X'_n))} |z-z'| \to 0.   
\end{align*}
\end{theorem}
\begin{proof}
Since $W$ and $W_n$ are Lipschitz, it follows that 
\begin{align*}
\norm{\partial_u T_W} \leq \norm{\nabla W}_{L^2} \\ 
\norm{\partial_u T_{W_n}} \leq \norm{\nabla W_n}_{L^2},  
\end{align*}
and therefore the $T_W,T_{W_n}$ are compact operators. 
Since the coefficients are bounded, it follows from the formulae above that the graphon-NTK and the induced graphon-NTKs are all compact operators on $L^2$. 
From \red{cite?!}, the convergence of the spectrum (and in fact also of all finite dimensional spectral projectors) follows. 
\end{proof} 
}

\comment{
\begin{corollary}
Let $\bbW$ be a Lipschitz continuous graphon and $X,X'$ be arbitrary, Lipschitz continuous graphon signals. Suppose that $\set{\bbA_n}_{n=1}^\infty$ be a sequence of generic GSO's sampled from $\bbW$ of dimension $\reals^{n \times n}$, and let $\bbW_n$ be the associated induced graphons and $X_n,X_n'$ be the induced graphon signals associated to $X,X'$ respectively.  
Then, for any $L$-layer WNN with $K$ fixed (finite) and fixed filter coefficients $\ccalH$, the associated NTK's converge in the operator norm: 
\begin{align*}
\lim_{n \to \infty}\norm{\Theta_W(X,X') - \Theta_{W_n}(X_n,X'_n)} = 0. 
\end{align*}
\label{corollary}
\end{corollary}
}

\comment{
\subsection{Convergence for different random graph models}

An approximate estimate depends on what kind of random graph is being generated.
For a weighted graph, Propositions 13 in the previous paper gives with probability at least $(1-2\chi_1)(1-\chi_2)$ for $n \geq 4/\chi_2$, 
\begin{align*}
\| X - X_n  \| & \leq \frac{A_x}{n} \log \left( \frac{(n+1)^2}{\log(1-\chi_1)^{-1}} \right) \|X\|
\end{align*}
(where $A_x$ is the Lipschitz constant of $X$) and 
\begin{align*}
\| W_n - W \| & \leq \frac{2A_w}{n} \log \left( \frac{(n+1)^2}{\log(1-\chi_1)^{-1}} \right) . 
\end{align*}
Let us now assume that $n$ is so large that
\begin{align*}
\frac{A_x}{n} \log \left( \frac{(n+1)^2}{\log(1-\chi_1)^{-1}} \right) < 1. 
\end{align*}
In the weighted graph case, it seems like we then have (assuming $X,X'$ have the same Lipschitz constants)

For stochastic graphs, the paper is a little less clear about exactly what the estimate is, but I feel like similar estimates, possibly up to another logarithm, hold due to Proposition 14.

\noindent \red{L: I like the proof steps above, linking with the bounds from the transferability paper which avoids all the ugly triangle inequalities and Cauchy-Schwarz. However, I think we still need to polish this section and add a bit more detail.}

\red{TO DO: (1) write estimates for $\| X - X_n  \|$ and $\| W - W_n  \|$  for stochastic graphs? (2) Evaluate a quantitative estimate for Eqn \ref{eq:NTK_matrix} for 2 layers by adding and subtracting (8 terms total) and using triangle inequality (since $U_l = \sigma X_l$ and we have bounds on $U_l$ for graph and graphon in terms of $W-W_n$. (3) Once we have this estimate, trivially write transferability  $\norm{\Theta_{n1} - \Theta_{n2}}$ using triangle inequality (Similar to how it's done below).}

\noindent \red{L: Will get to (1) soon.}

}

\section{Numerical Results}
\label{sec:numerical}


\begin{table}[t]
    \caption{Movie recommendation results. Relative difference between the test MSE achieved by the GNTK fitted on the $n$-node graph when used for prediction the $n$-node graph, and the test error achived by the same GNTK when used for predicition on the $N$-node graph, where $N=1000$.}
    \centering
    \begin{tabular}{l|ccccc}
    \hline
        Number of nodes & {100} & {200} & {300} & {400} & {500} \\ \hline
        Mean (\%) & $6.80$ & $0.78$ & $0.09$ & $0.03$ & $0.03$ \\ 
        Std. dev. (p.p.) & $3.51$ & $1.16$ & $0.03$ & $0.01$ & $0.02$ \\ \hline
    \end{tabular}
    \label{tb:movie_results}
\end{table}


\begin{figure*}[t]
\centering
\includegraphics[width=0.33\linewidth]{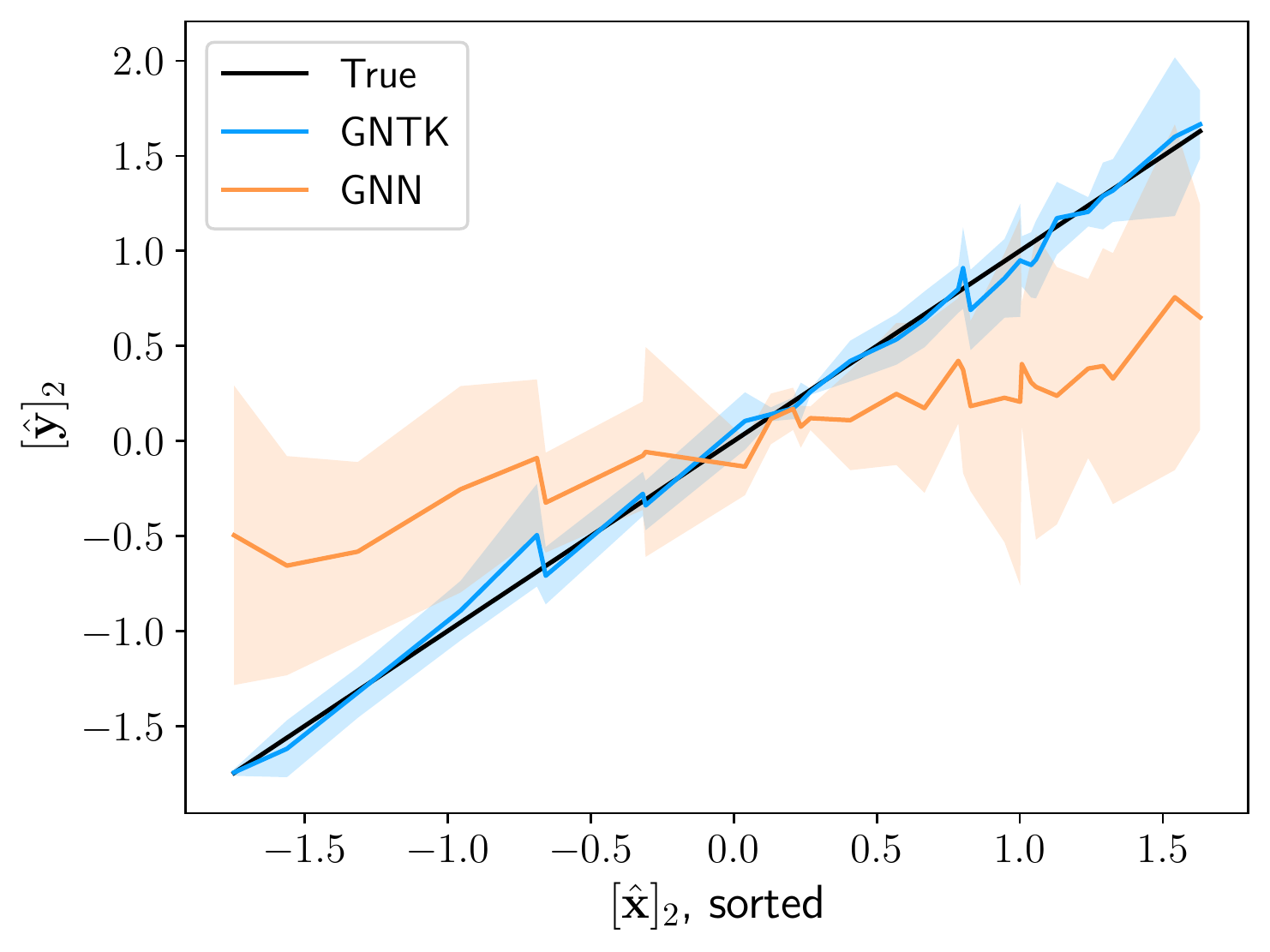}
\includegraphics[width=0.33\linewidth]{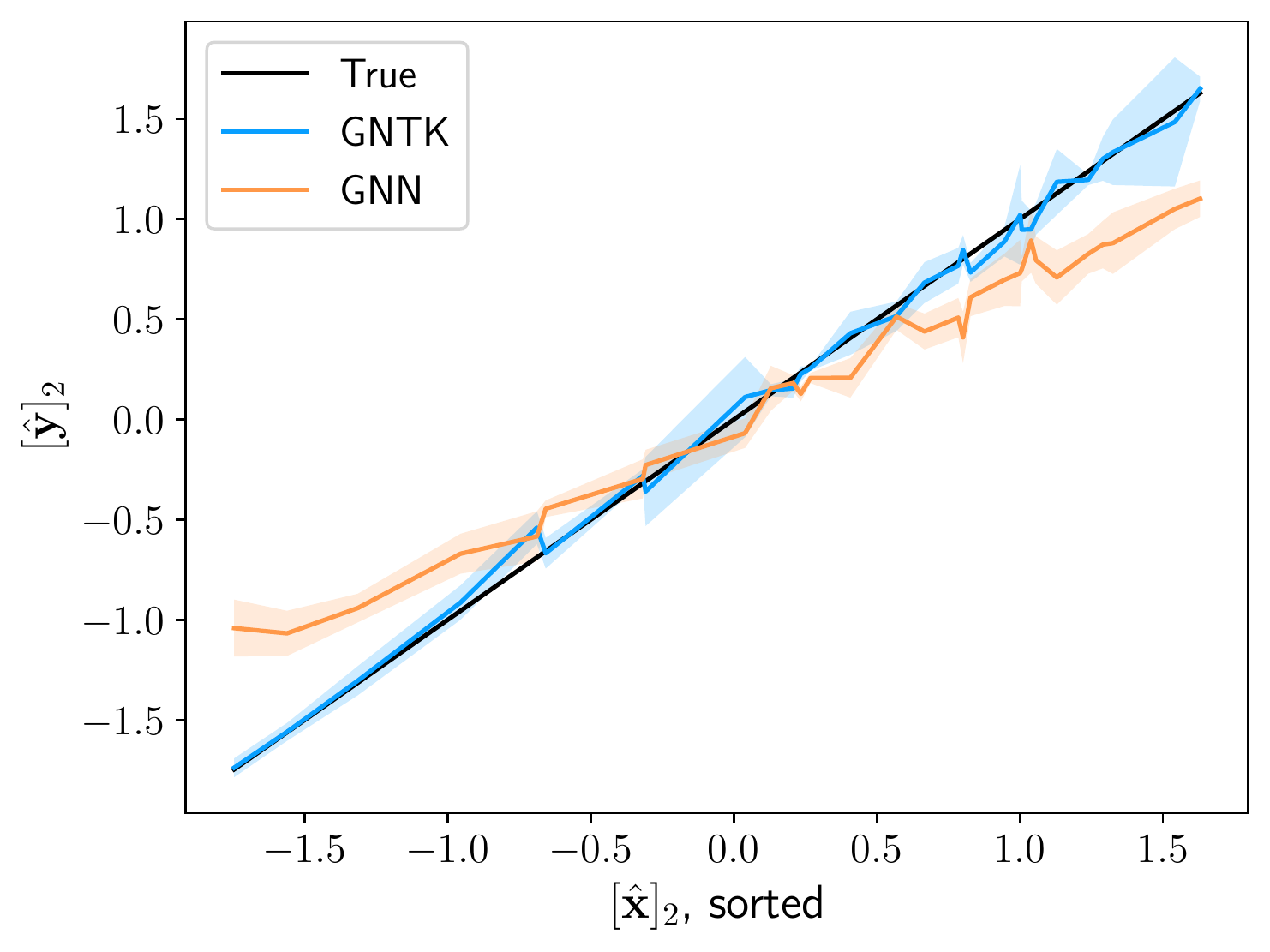}
\includegraphics[width=0.33\linewidth]{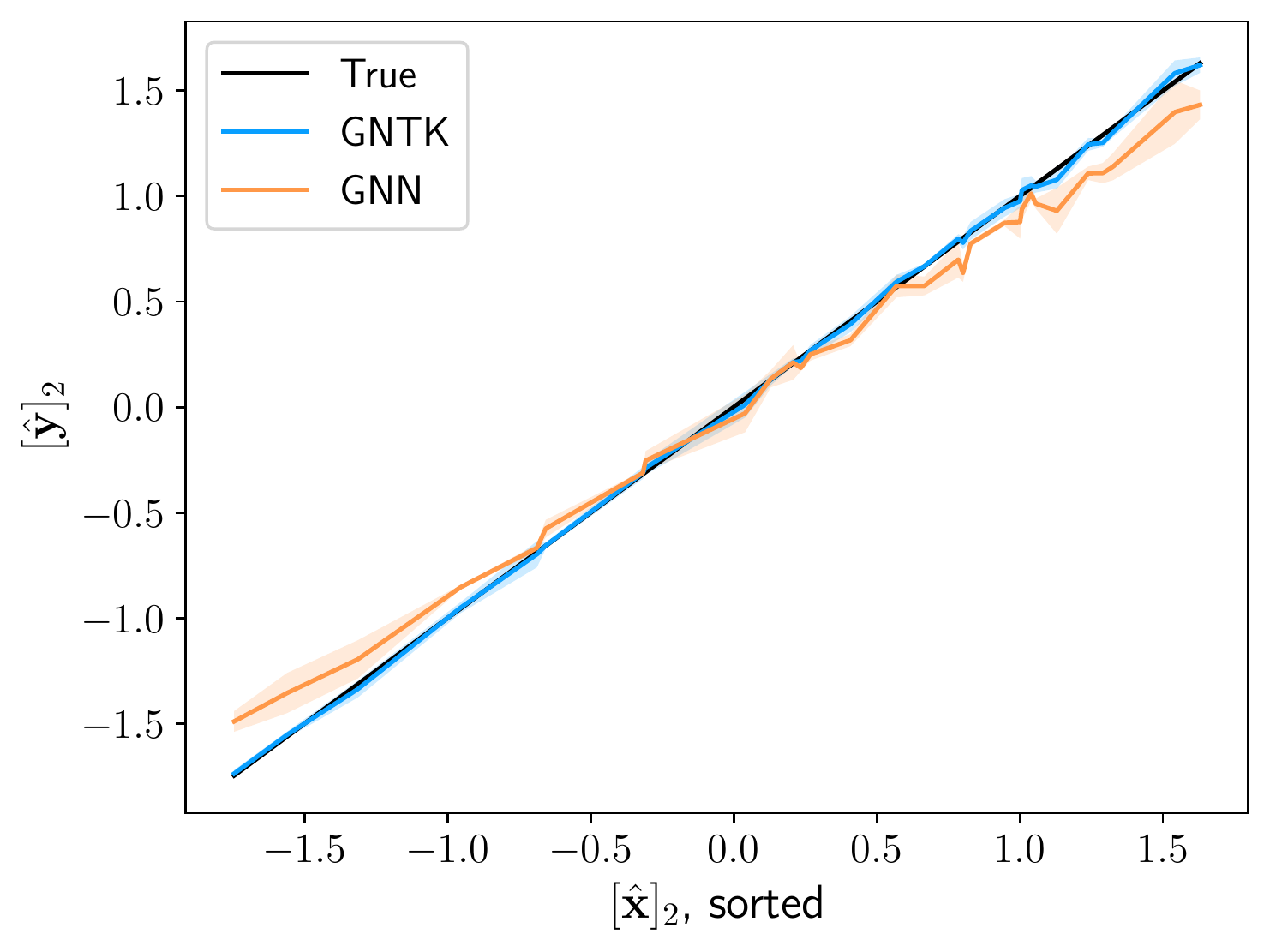}
\caption{Projections of the inputs and outputs of the GNTK and GNN, and of the target or true labels onto the second eigenvector of the adjacency matrix of a $80$-node SBM graph for widths $F=10$ (left), $F=50$ (center) and $F=250$ (right).}
\label{fig:width}
\end{figure*}

In the following, we illustrate the convergence of the GNTK to the WNTK on a simulated opinion dynamics model on random graphs; on a movie recommendation graph; and on citation networks (Sec. \ref{sbs:sims_conv}). In the opinion dynamics problem, we further vary the width of the network to analyze the large width behavior of the GNTK and how it relates to the behavior of the GNN (Sec. \ref{sbs:wide_sims}). We conclude with an example of how the convergence of the GNTK can be used in practice to estimate the eigenvalues of the GNTK, and thus the speed of GNN training along its principal components, on large-scale graphs (Sec. \ref{sbs:eig_sims}).

\noindent \textbf{Opinion dynamics.} This is a node-level task modeling the outcomes of a mathematical model for studying the evolution of ideologies, affiliations, and opinions in society \cite{lorenz2007continuous}, including topics of important practical interest such as political ideologies and misinformation spread. On an undirected $n$-node graph $\bbG = (\ccalV,\ccalE)$, we consider an opinion dynamics process $\bbx_t \in \reals^n$. The node data $[\bbx_t]_i$ is the opinion of individual $i$ under a standard bounded-confidence model described by \citet{hegselmann2002opinion}:
\begin{equation} \label{eqn:op_dyn}
[\bbx_t]_i = \sum_{j \in \ccalS_{i,t-1}} c [\bbx_{t-1}]_j.
\end{equation}
where  $c>0$ is the so-called influence parameter and $\ccalS_{i,t-1} = \ccalN_i \cap \ccalX_{i,t-1}$ is the intersection of $\ccalN_i = \{j \in \ccalV\ |\ (i,j) \in \ccalE\}$, the neighborhood of $i$, and $\ccalX_{i,t} = \{j \in \ccalV\ |\ |[\bbx_t]_j-[\bbx_t]_i| \leq \epsilon\}$, the set of nodes whose opinion at time $t$ diverges by at most $  0\leq \epsilon \leq 1$ of the opinion of node $i$. This model reduces to the classic De-Groot \cite{degroot1974reaching} opinion model for $\epsilon=1$. We fix $\epsilon=0.3$ and $c=0.1$. 

After drawing the initial opinions $\bbx_0$ from a Gaussian distribution with $\bbmu = \boldsymbol{0}$ and $\bbSigma = 2\bbI$, we run the process \eqref{eqn:op_dyn} until convergence (for at most $T=1000$ iterations) to obtain the final opinions $\bbx_T$. The goal of the learning problem is then to predict $\bby=\bbx_T$ from $\bbx=\bbx_0$. We fix the training set size to $300$ samples, and use $30$ samples for both validation and testing. 

\noindent \textbf{Movie recommendation.} In this problem, the graph is a movie similarity network, the data are existing user ratings (ranging from 1 to 5), and the task is to use the existing ratings to fill in the missing ratings for movie recommendation. The movie similarity graph is a dense correlation graph built by computing pairwise correlations between the rating vectors of the movies rated in the MovieLens-100k dataset \cite{harper16-movielens}. We only consider movies with at least 10 ratings, and the large graph has $N=1000$ nodes sampled at random from the set of admissible movies. The learning problem is to predict the ratings of the movie ``Contact''. We use a 90-10 train-test split of the data, and build the graph using only the training samples. Due to memory limitations, we only use 25\% of the users (samples) to compute the GNTK.

\noindent \textbf{Citation networks.} The third problem we consider is a node classification problem on the Cora, CiteSeer and PubMed networks, where nodes represent documents and undirected edges represent citations between papers in either direction.
The nodes a bag-of-words representations are grouped in $\bbX \in \{0,1\}^{n \times F}$ and each node is associated with one of $C$ classes. We consider the networks and {train-test-splits} from the full distribution \cite{bojchevski2017deep} available in PyTorch Geometric, but only sample $F=1000$ features for CiteSeer and $N=10000$ nodes for PubMed due to memory limitations. {Through studying citation networks, we illustrate the empirical manifestation of our results on real-world graphs. Although citation networks are not best modeled by graphon limits, they certainly have limits, and our results suggest that convergence is a global property of GNTKs and graph limits.}

\noindent \textbf{Architectures and experiment details.} In all experiments, the GNNs have $L=1$ layer \eqref{eqn:gcn_layer} with ReLU nonlinearity followed by a perceptron layer. For opinion dynamics, $K=2$, and for movie recommendation, $K=5$. In both cases, we consider the MSE and fit the GNTK using linear regression. In the citation network experiment, the GNN architecture has $K=2$ and includes a softmax layer followed by argmax; and we consider the cross-entropy (CE) and fit the GNTK using logistic regression. All reported results are averaged over $5$ and $10$ realizations for opinion dynamics and movie recommendation/citation networks respectively. Additional architecture and experiment details are listed in each subsection. The code can be found in \href{https://github.com/luanaruiz9/wntk.git}{this} repository. All experiments were run on a NVIDIA RTX A6000 GPU.

\subsection{Convergence} \label{sbs:sims_conv}

To visualize the convergence of the GNTK, we fit a GNTK with $F_1=10$ in the opinon dynamics and citation network experiments, and $F_1=16$ in the recommendation experiment, on the training set of a small $n$-node graph; and transfer this GNTK, without retraining, to predict the outputs on the test set of both the same $n$-node graph and a larger $N$-node graph. To visualize convergence, we plot in Fig. \ref{fig:convergence} the absolute difference between the test errors attained on the $n$-node and the $N$-node graph (normalized by the test error on the $N$-node graph) as a function of the size of the smaller graph $n$, for both the opinion dynamics (left) and node classification (right) experiments. The results of the recommendation experiment are reported in Table \ref{tb:movie_results}.


\begin{figure}[t] 
\centering
\includegraphics[width=\linewidth]{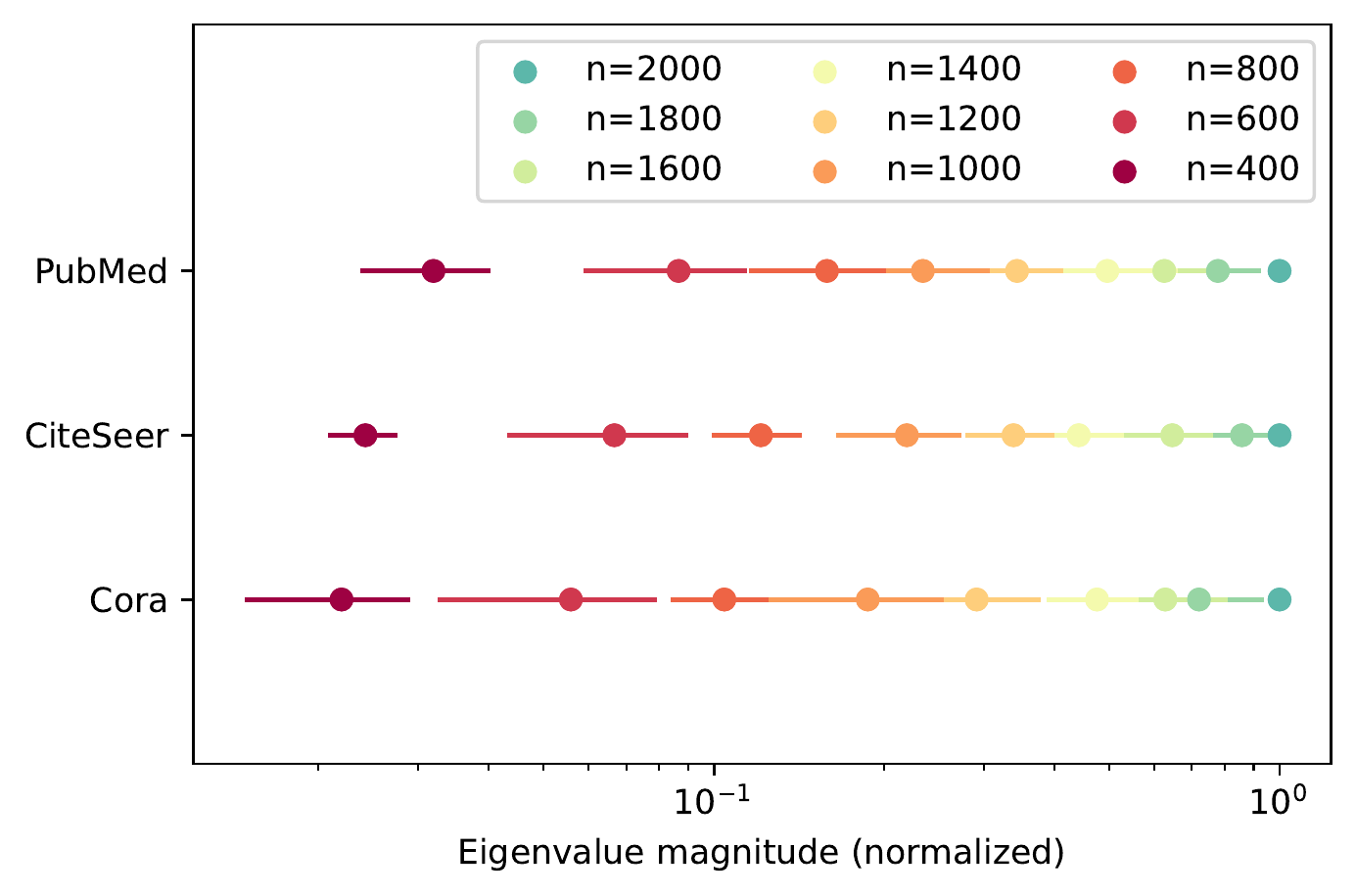}
\caption{The leading eigenvalue of the GNTK with width $F=10$ for Cora, CiteSeer and PubMed as a function of the graph size $n$.}
\label{fig:eigs}
\end{figure}

In opinion dynamics, $N=300$. We consider two types of graphs, corresponding to two graphon families: (a) a symmetrized geometric $k$-nearest neighbor graph where nodes are drawn at random from a $50 \times 50$ square and $k=n/10$; and (b) a stochastic block model (SBM) graph with intra-community probability $p=0.1$ and inter-community probability $q=0.05$. As $n$ increases from $20$ to $100$, we observe that the difference between the MSEs for the $n$-node and the $N$-node graphs decrease steadily 
from $\sim 70$\% to $\sim 20$\% on the geometric graphs and from $\sim 90$\% to $\sim 40$\% on the SBM graphs. 

In the recommendation experiment, $N=1000$. As $n$ increases from $100$ to $500$, we see from Table \ref{tb:movie_results} that the average test MSE difference on the large graph decreases from $6.8$\% to $0.3$\%, and so does the standard deviation, thus confirming our theoretical findings. In particular, $n=300$ is enough to achieve an MSE difference of less than $1$\%.

In the node classification experiment, $N=2708$ for Cora, $N=3327$ for CiteSeer, and $N=10000$ for PubMed (sampled from the $19717$ nodes in the original due to memory limitations). For all datasets, the difference between the CE loss for the $n$-node and the $N$-node graph also decreases as $n$ increases.
This can be interpreted to mean that node classification on these citation networks depends largely on local information, so beyond a critical size $n$, the transference error saturates. An interesting point to make is that although these citation networks are not dense (i.e., are not best modeled by a graphon), we still see the empirical manifestation of our theoretical results.

\subsection{Wide Network Behavior} \label{sbs:wide_sims}

Next, we analyze the effect of width on both the GNN and the GNTK when they are trained/fitted on a small graph and transferred to a large graph. 
From the NTK analysis by \citet{jacot2018neural}, as the GNN width increases we expect (i) the GNTK and the GNN to exhibit less variance over multiple weight initializations and (ii) the GNN outputs to approach those of kernel regression with the GNTK. 

For the opinion dynamics experiment on SBMs, we consider three GNNs with widths (number of features) $F^{(1)}_1=10$, $F^{(2)}_1=50$, and $F^{(3)}_1=250$ supported on the same $80$-node graph. Using their initial weights, we construct the corresponding GNTKs. We train the three GNNs by minimizing the MSE loss over $20$ epochs and with batch size $32$, using ADAM \cite{kingma17-adam} with learning rate $1e^{-3}$ and weight decay $5e^{-3}$. Simultaneously, we fit the GNTKs to the training set, 
and then transfer both the GNNs and the corresponding GNTKs to the $N$-node graph and compute their outputs on the test set.

In Fig. \ref{fig:width}, we plot the projection of the outputs $\bby$ onto the second graph eigenvector $\bbv_2$, $[\hby]_2=\bbv_2^T\bby$ against the projections of the inputs onto the same vector $[\hbx]_2 = \bbv_2^T\bbx$ (sorted in ascending order) for both the GNNs and the GNTKs for $F^{(1)}_1=10$ (left), $F^{(2)}_1=50$ (center), and $F^{(3)}_1=250$ (right). The second graph eigenvector was chosen as it captures the community structure, however, note that the same behavior was observed upon projecting onto other eigenvectors, as shown in Appendix \ref{appendix:wide}. The results are averaged over $5$ GNN initializations, with the solid lines representing the mean and the shaded areas representing the standard deviation. The behavior is as expected: as the width increases, the GNN and the GNTK have smaller variance in behavior for different weight initializations, and the GNN and GNTK curves align. In Appendix \ref{appendix:wide}, we include additional plots displaying the mean and variance over different initializations as we vary the graph size.

\subsection{Application: Eigenvalue Convergence} \label{sbs:eig_sims}

In this section, we elucidate an important application of the convergence of the GNTK. \citet{jacot2018neural} shows that the convergence of kernel gradient descent follows the kernel principal components. Therefore, if the GNN is sufficiently wide, and $f(\bbX_n;\bbA_n,\ccalH_0)-f(\bbX_n;\bbS_n,\ccalH^\star)$ is aligned with the $i$th GNTK principal component, we can expect gradient descent to converge with rate proportional to $\lambda_i$, the $i$th eigenvalue of the GNTK\footnote{Theoretically, constancy of the NTK in the infinite-width limit and equivalence between GNN training and kernel regression only holds for architectures with linear output layer \cite{liu2020linearity}. Although this is not the case of the GNNs used here, GNTKs associated with GNNs with nonlinear output layers seem to exhibit constancy empirically \cite{sabanayagam2021new}}. 

Recall that we prove, in Thm. 2, that the spectrum of the GNTKs converge in the graphon limit. In this context, a simple application of the convergence of the GNTK spectrum is as follows: say that we know that $f(\bbX_N;\bbA_N,\ccalH_0)-f(\bbX_N;\bbA_N,\ccalH^\star)$ is aligned with the $p$th GNTK eigenvector, and we want to estimate its speed of convergence, but the kernel $\Theta(\bbX_N,\bbX_N';\bbA_N,\ccalH)$ is too expensive to compute. We could in this case sample a graph $\bbG_n$, $n \ll N$ from the induced graphon $\bbW_N$ and calculate $\Theta(\bbX_n,\bbX_n';\bbA_n,\ccalH)$. From Thm. \ref{theorem2}, $\lambda_{n,p}$ converges to $\lambda_{N,p}$. Hence, for large enough $n$, $\lambda_{n,p}$ should provide a good approximation of $\lambda_{N,p}$.

To illustrate this empirically, in Fig. \ref{fig:eigs} we plot the dominant eigenvalues of the GNTK associated with a GNN with $F_1=10$ features supported on graphs of increasing size $n$ sampled from the Cora, CiteSeer, and PubMed networks (averaged over 10 initializations). As $n$ grows, we see that the difference between the dominant eigenvalues of consecutive GNTKs reduces. I.e., for each citation dataset, the value of the dominant eigenvalue converges as $n$ increases, indicating that the speed of convergence of the GNN along the dominant GNTK eigenvector converges as the graph grows. For all but the PubMed dataset, the $1600$-node graph gives an approximation of $\lambda_{N,1}$ that is over $70$\% good.

\section{Conclusions}
\label{sec:conclusion}

In this paper, we define WNTKs as the limiting objects of GNTKs, and study how the GNTK evolves as the underlying graph of the corresponding GNN grows. We show that GNTKs converge to the WNTK and that their spectra also converges, thus providing theoretical insight into how the learning dynamics of a GNN evolves as an important dimension---the graph size---grows. In practice, these convergences imply that one can transfer the GNTK of a smaller graph to solve the same task on a larger graph without any further optimization, with theoretical guarantees of performance and insight on the rates of learning along eigendirections that converge as the graph grows.
These results were demonstrated through simulations on synthetic and real-world graphs. To conclude, a limitation of this work is that graphons are only good models for dense graphs. In future work, we plan to extend our results to more general graph models, and to better understand the relationship between the spectra of the graph and of the GNTK.



\bibliography{myIEEEabrv,bib_dissertation,example_paper}
\bibliographystyle{icml2022}

\newpage
\appendix
\onecolumn
\section{Proof of GNTK Convergence}
\label{appendix:convergence_proof}

\begin{proof}[Proof of Thm. \ref{theorem1}]

To simplify notation, we will write $\Theta(X,X';\bbW,\ccalH)=\Theta(X,X')$ and $\Theta(X_n,X_n';\bbW_n,\ccalH)=\Theta(X_n,X_n')$. Let $Y$ be an arbitrary $L^2$ signal. Using the triangle inequality,
\begin{align} \label{eq:NTK_matrix}
\begin{split}
  ||\Theta(X,X')Y  - \Theta(X_n,X_n')Y || &\leq \sum_k || \sigma'(U_L) T_{W}^{(k)} X_{L-1}  \langle\sigma'(U_L') T_{W}^{(k)} X_{L-1}',Y\rangle  \\
  & \quad - \sigma'(U_{L,n})T_{W_n}^{(k)} X_{n,L-1} \langle\sigma'(U_{L,n}') T_{W_n}^{(k)} X_{L-1,n}'Y\rangle || + \ldots
  \end{split}
\end{align}
where there are $L$ terms for each $k$.

Since $\sigma$ is Lipschitz, from the formula  we see that it suffices to prove the following convergence statements for all $k \in {1,...,K-1}$ as $n \to \infty$:
\begin{align}
  \|U_{l,n} &- U_l\| \to 0 \label{eq:conv:1} \\
  \|U_{l,n}' &- U'_l\| \to 0 \label{eq:conv:2}
\end{align}
as well as
\begin{align}
  \norm{T_{W_n}^{(k)} - T_{W}^{(k)}} \to 0 \label{eq:conv:3}\\
  \norm{T_{h_{l,n}} - T_{h_{l}}} \to 0 \text{.} \label{eq:conv:4}
\end{align}
 \comment{These are proved for a more general class of (analytic) filters 
 (i.e. with $K \to \infty$ in a controlled way) using Fourier transforms in Appendix D of the previous paper (cite).} 
 Here, we prove these statements for each finite $K$ using mathematical induction.
 \comment{\red{Luana can you rewrite these lines: Still, our result also holds for general analytic filters, because any polynomial in $[0,1]$ is Lipschitz---ensuring the convergence (30) from (cite))---, and, provided that the $H_{n,l}$ are Lipschitz, (31) converges per (cite).} \blue{Is that so clear? We would need to quantify the rate of convergence in terms of both $k$ and $n$ in a more careful way I think, so a longer explanation I think might be warranted to say we can do the case $K \to \infty $. It looks like if $k^2/n \to 0$ then we can send $K \to \infty$ along with $n \to \infty$, provided the $h_k$'s aren't doing anything too crazy}. 
 }
For fixed, finite $K$, \eqref{eq:conv:3} implies  \eqref{eq:conv:4}. The convergence in \eqref{eq:conv:1} and \eqref{eq:conv:2} are essentially the same, so let us focus on the first without loss of generality. 
Expanding the first layer, we have
\begin{align*}
  U_{1,n} - U_1 &= \sum_{k} h_{1,k} (T_{W_n}^{(k)} X_n - T_{W}^{(k)} X) = \sum_{k} h_{1,k} (T_{W_n}^{(k)} X_n - T_{W}^{(k)} X  + T_{W}^{(k)}  X_n - T_{W}^{(k)}  X_n)  \\
  & = \sum_{k} h_{1,k} ((T_{W_n}^{(k)} - T_W^k) X_n + T_{W}^{(k)} (X_n - X)).
\end{align*}
Therefore, by the triangle inequality and the definition of the operator norm, we have
\begin{align}
 \|U_{1,n} - U_1\| & \leq \sum_{k} |h_{1,k}| 
 \|T_{W_n}^{(k)} - T_W^k\| \|X_n\| + \sum_k |h_{1,k}| \norm{T_{W}^{(k)}} \| X_n - X \| \\
  & \leq \sum_{k} \max(|h_{1,k}|) \| T_{W_n}^{(k)} - T_W^k \|  \| X_n \| + \sum_{k} \max(|h_{1,k}|) \| T_{W}^{(k)}\| \|X_n - X\| .
  \label{eq:u1conv}
 \end{align}

Looking at layer $l$, we have
\begin{align*}
U_{l,n} - U_l = \sum_{k} h_{l,k} (T_{W_n}^{(k)} \sigma(U_{l-1,n}) - T_{W}^{(k)} \sigma(U_{l-1})). 
\end{align*}
From \ref{eq:u1conv}, using mathematical induction, we have, for arbitrary $l$, 
\begin{align}
\|U_{l,n} - U_l \| &\leq (\| X_n \| \|X_n - X\| + \| T_{W_n}^{(k)} - T_W^k \| ) \cdot (\max(|h|) k)^l \\
& + \sum_{j=1}^{l-1}(\max(|h|) k)^{l-1-j} \max(|h|) \| T_{W_n}^{(k)} - T_W^k \| 
\label{eq:ulconv}
\end{align}
where $h$ is the max of $h_{l,k}$ over all $l$ and $k$, and so convergence of $T_{W_n}^{(k)} \to T_{W}^{(k)}$ as in \eqref{eq:conv:3} and $X_n \to X$ implies convergence of $U_{1,n} \to U_1$ and $U_{l,n} \to U_l$.
Since the nonlinearity $\sigma$ is Lipschitz, convergence of \eqref{eq:conv:2} follows from convergence of \eqref{eq:conv:1}.

We now focus on proving \eqref{eq:conv:3}.
By Cauchy-Schwarz, 
\begin{align*}
  \norm{(T_{W_n} - T_{W})Y}^2 & = \int_0^1 \left( \int_0^1 (\bbW_n(u,v) - \bbW(u,v)) Y(v) dv \right)^{2} du \\
  &\leq \left(\int_0^1 \int_0^1 (\bbW_n(u,v) - \bbW(u,v))^{2} dv du \right) ||Y||^2.
\end{align*}
Therefore, taking the square root, dividing by $||Y||$, and using \eqref{eq:normdef}, we get
\begin{align}
\norm{T_{W_n} - T_W} \leq \norm{\bbW_n - \bbW}_{L^2}.
\label{eq:Twconv}
\end{align}
It follows that convergence of the induced graphons $\bbW_n$ to the limiting graphon $\bbW$ implies \eqref{eq:conv:3} for $k=1$.
For higher iterates we note 
\begin{align*}
 \norm{(T_{W_n}^{(k)} - T_{W}^{(k)})Y}^2  &\leq {\int_0^1 \left( \int_0^1 (\bbW_n T_{W_{n}}^{(k-1)} Y(v) - \bbW T_{W}^{k-1}Y(v) )  dv \right)^{2} du} \\
  & \leq {\int_0^1 \left( \int_0^1 \bbW_n (T_{W_{n}}^{(k-1)} - T_{W}^{(k-1)})Y(v)  dv \right)^{2} du }  + { \int_0^1 \left( \int_0^1 (\bbW_n - \bbW)T_{W}^{(k-1)}Y(v)  dv \right)^{2} du }. 
\end{align*}
Therefore, by the same Cauchy-Schwarz argument above, $\norm{T_{W_n}^{(k)} - T_{W}^{(k)}}$ is upper bounded by 
\begin{align*}
\norm{\bbW_n} \norm{T_{W_n}^{(k-1)} - T_W^{(k-1)}} + \norm{\bbW_n - W} \norm{T_W^{(k-1)}}
\end{align*}
and so convergence $\bbW_n \to \bbW$ implies the convergence of all the iterates $T^{(k)}_{W_n} \to T^{(k)}_{W}$ by induction. In fact, one can write 
\begin{align}
\norm{T_{W_n}^{(k)} - T_{W}^{(k)}} \leq k \norm{\bbW_n - \bbW}_{L^2} 
\end{align}
\label{eq:Twkconv}
since $\norm{T_W^{(k-1)}} \leq 1$ and $\norm{\bbW_n}\leq 1 $.

From the proof one can, with some simple algebra, compute a quantitative estimate of $\norm{\Theta(X_n,X_n') - \Theta(X,X')}$ in terms of $\norm{\bbW_n - \bbW}$, $\norm{X_n - X}$, $\norm{X}$ and $\norm{X'}$. Precisely, by adding and subtracting cross-terms for each term in \eqref{eq:NTK_matrix} and using the triangle inequality, we get
\begin{align}
\norm{\Theta(X_n,X_n') - \Theta(X,X')} \leq C (K^{4+L} \norm{\bbW_n - \bbW} + K^{2+L}\norm{X_n - X}),
\label{eq:nonasymp}
\end{align}
where $C$ is a constant in terms of $\| X \|, \| X' \|, max(| h |), l$.
For different choices of graphs, one can then compute these quantities for an explicit estimate of the convergence bound. \cite{ruiz2021transferability}[Appendix B] provides quantitative estimates on $\norm{\bbW_n - \bbW}$ and $\norm{X_n - X}$ for random graph models such as so-called template graphs, weighted graphs and stochastic graphs.

\end{proof}

\section{Proof of Spectrum Convergence} \label{appendix:eigenvalue_proof}
\begin{proof}[Proof of Theorem \ref{theorem2}]
The proof relies on \cite{anselone1968spectral}[Prop. 7.1] which proves that if a sequence of compact operators $T_n$ converges to a compact operator $T$ in the appropriate operator norm, the spectrum of $T_n$ converges to the spectrum of $T$. Additionally, from \cite{anselone1968spectral}[Thm. 6.3 and Prop 7.1], for every $p$, the eigenspace associated with the $p^{th}$ eigenvalue of $T_n$ also converge to the corresponding eigenspace of $T$. From Thm. \ref{theorem1}, we know that $\bbTheta_n \to \bbTheta$. It remains to show that $\bbTheta_n$ and $\bbTheta$ are compact.

For any graphon $\bbW$, the operators $T_W^{(k)}$ and $T_H$ [cf. \eqref{eqn:gen_graphon_convolution}] are compact as they are linear compositions of compact operators with bounded linear weights. From \eqref{eqn:wntk}, we thus conclude $\|\Theta(X,X';\bbW,\ccalH)\|< \infty$. Therefore, for $Y \in L^2([0,1])$, $\|\Theta(X,X';\bbW,\ccalH)Y\|< \infty$, by which we conclude that $\bbTheta$ is compact. 
\end{proof}

\begin{figure*}[ht] 
\centering
\includegraphics[width=0.32\textwidth]{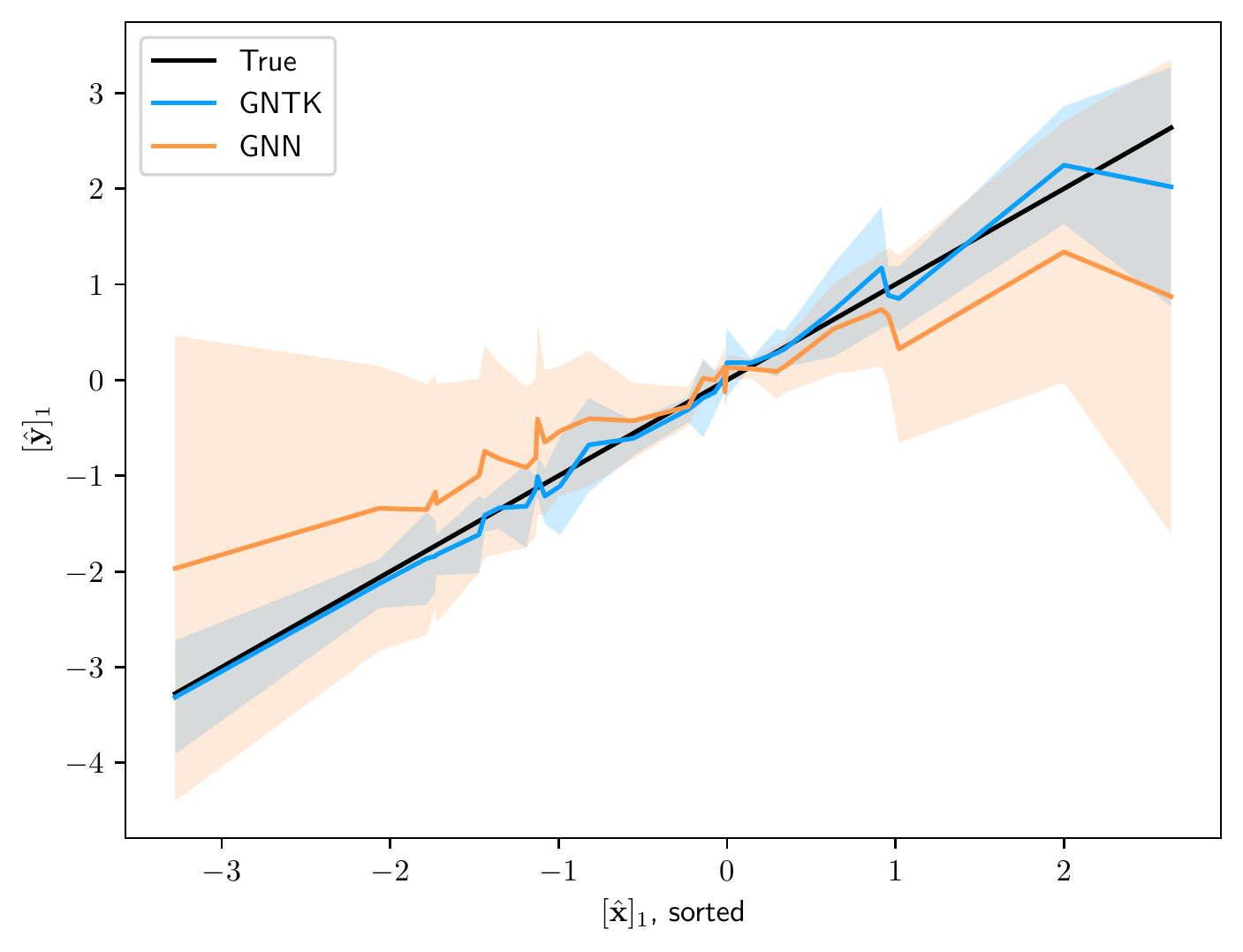}
\includegraphics[width=0.32\textwidth]{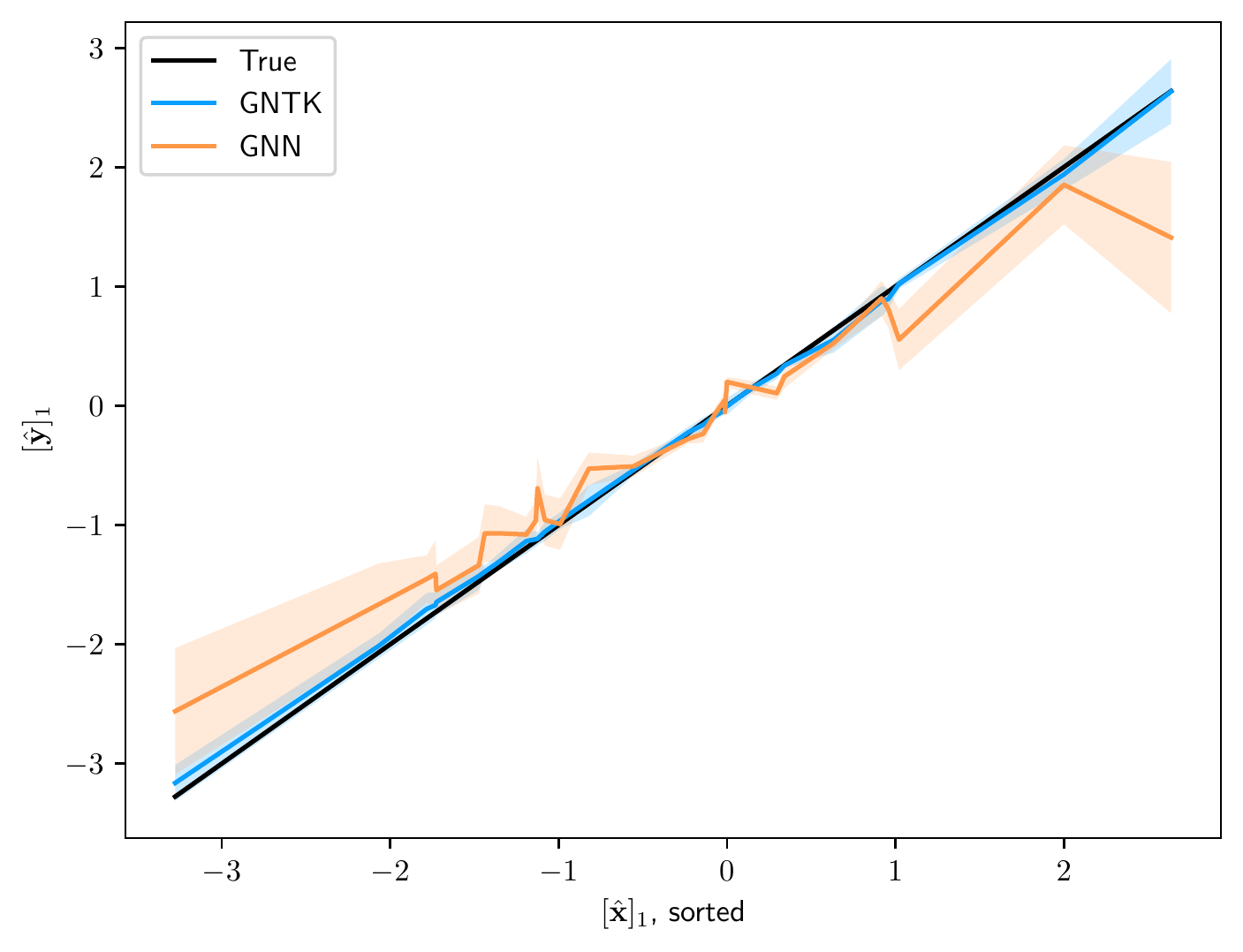}
\includegraphics[width=0.32\textwidth]{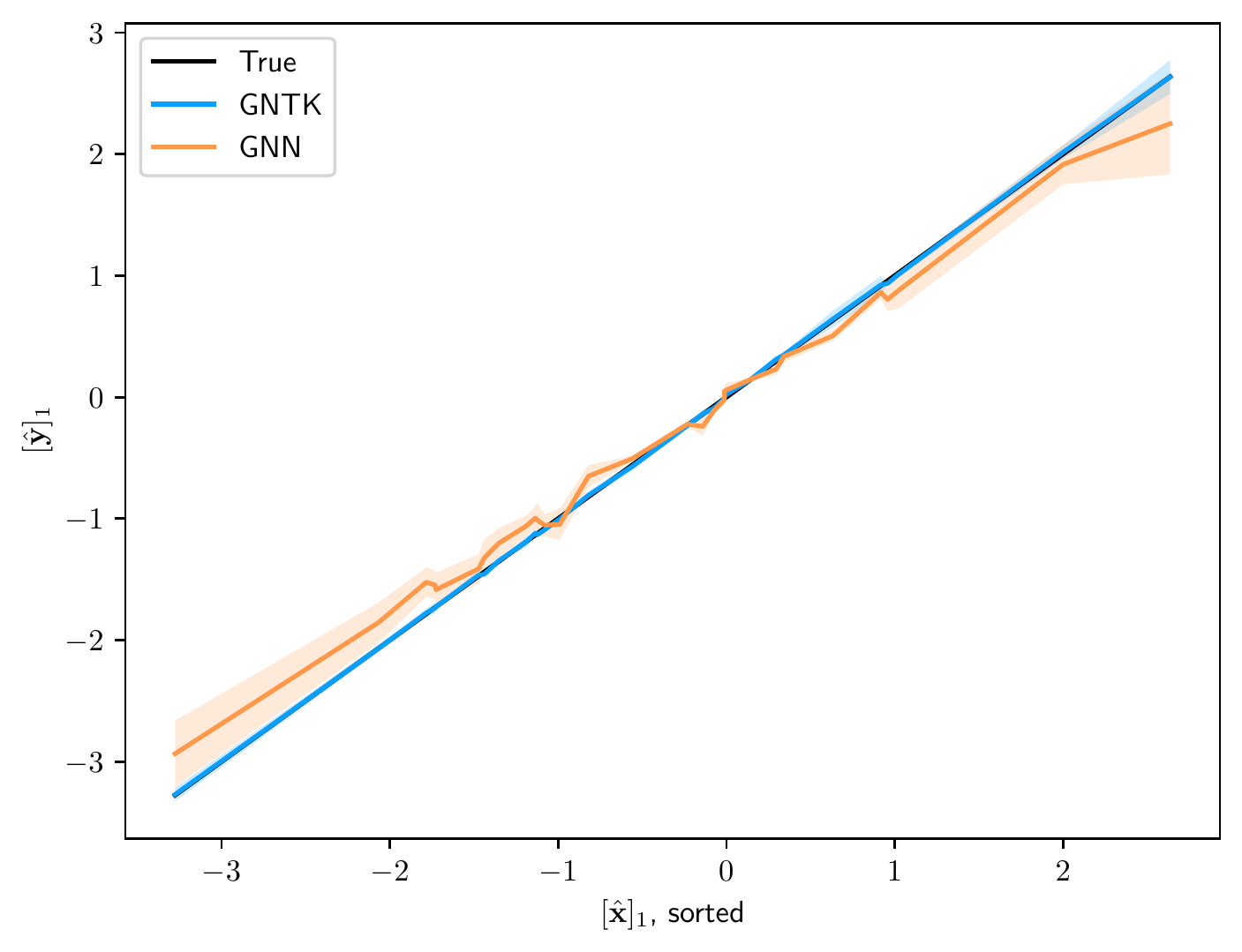}
\includegraphics[width=0.32\textwidth]{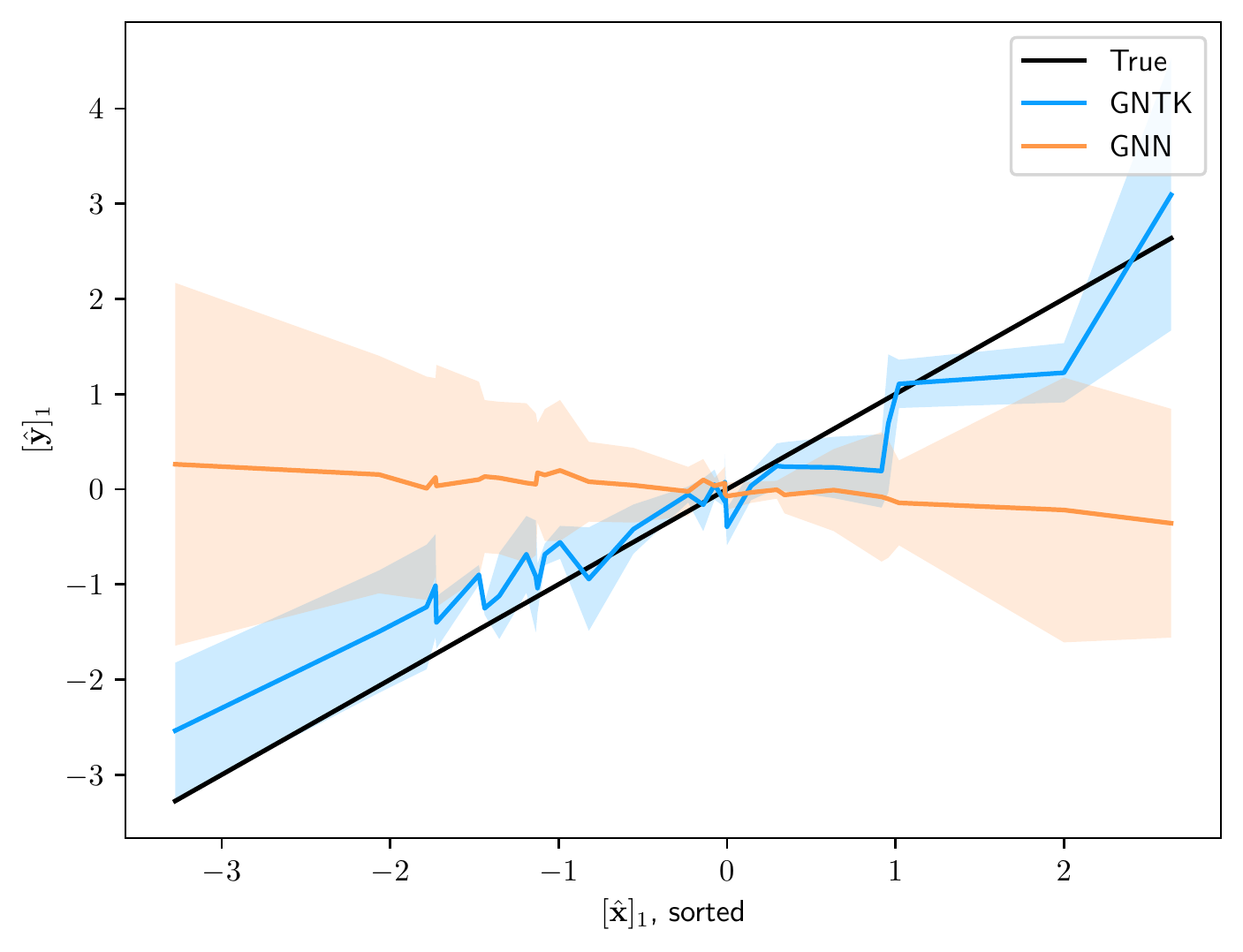}
\includegraphics[width=0.32\textwidth]{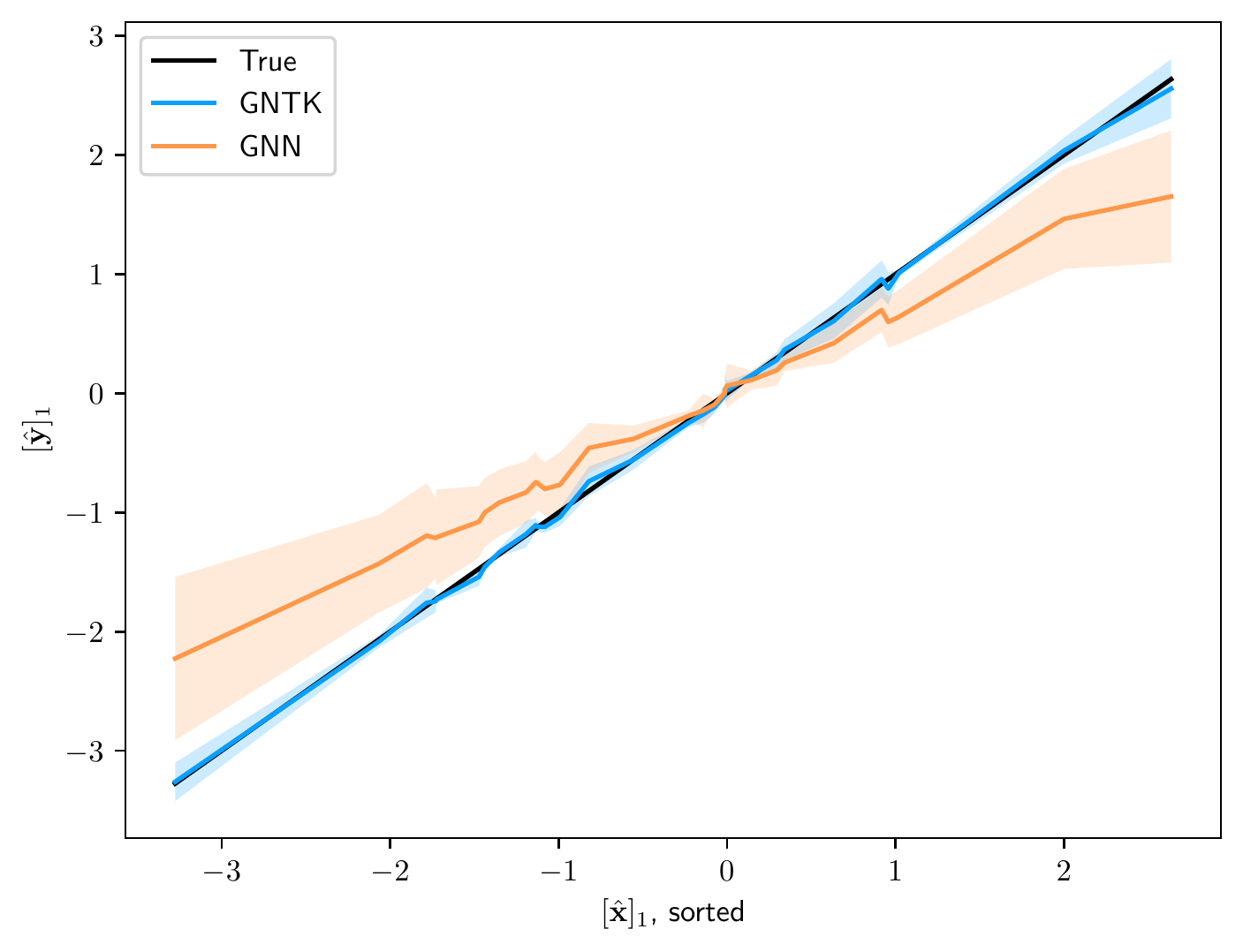}
\includegraphics[width=0.32\textwidth]{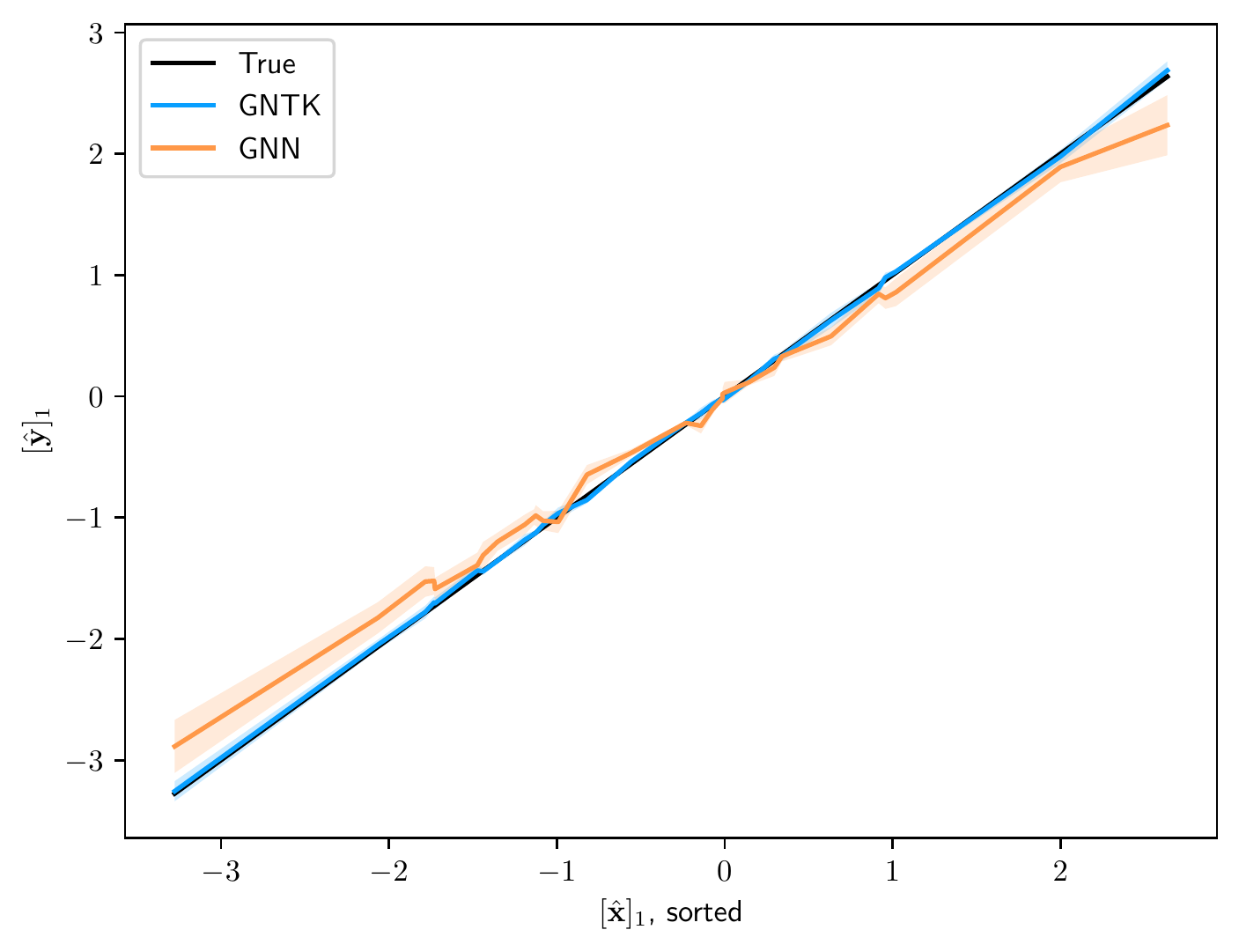}
\includegraphics[width=0.32\textwidth]{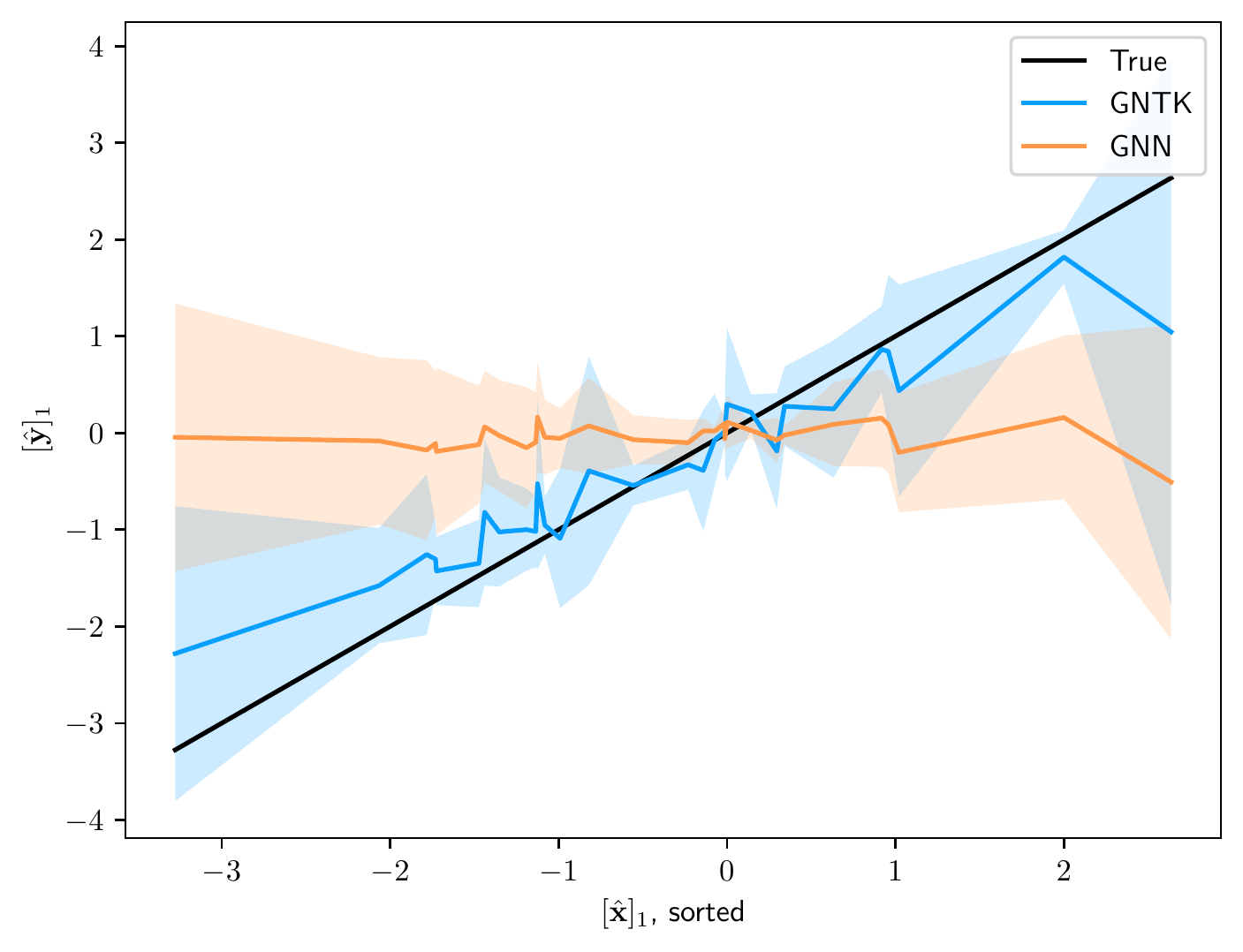}
\includegraphics[width=0.32\textwidth]{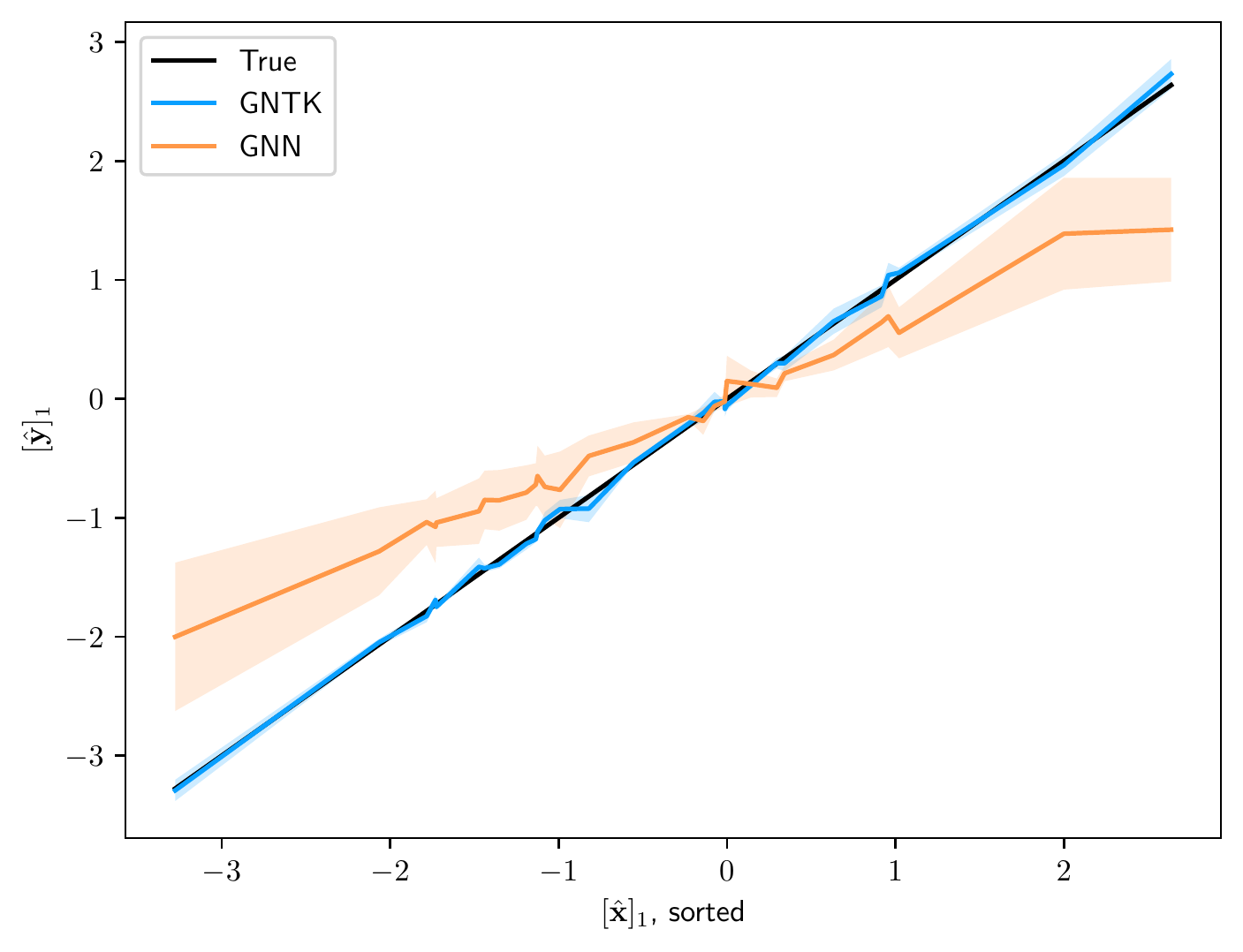}
\includegraphics[width=0.32\textwidth]{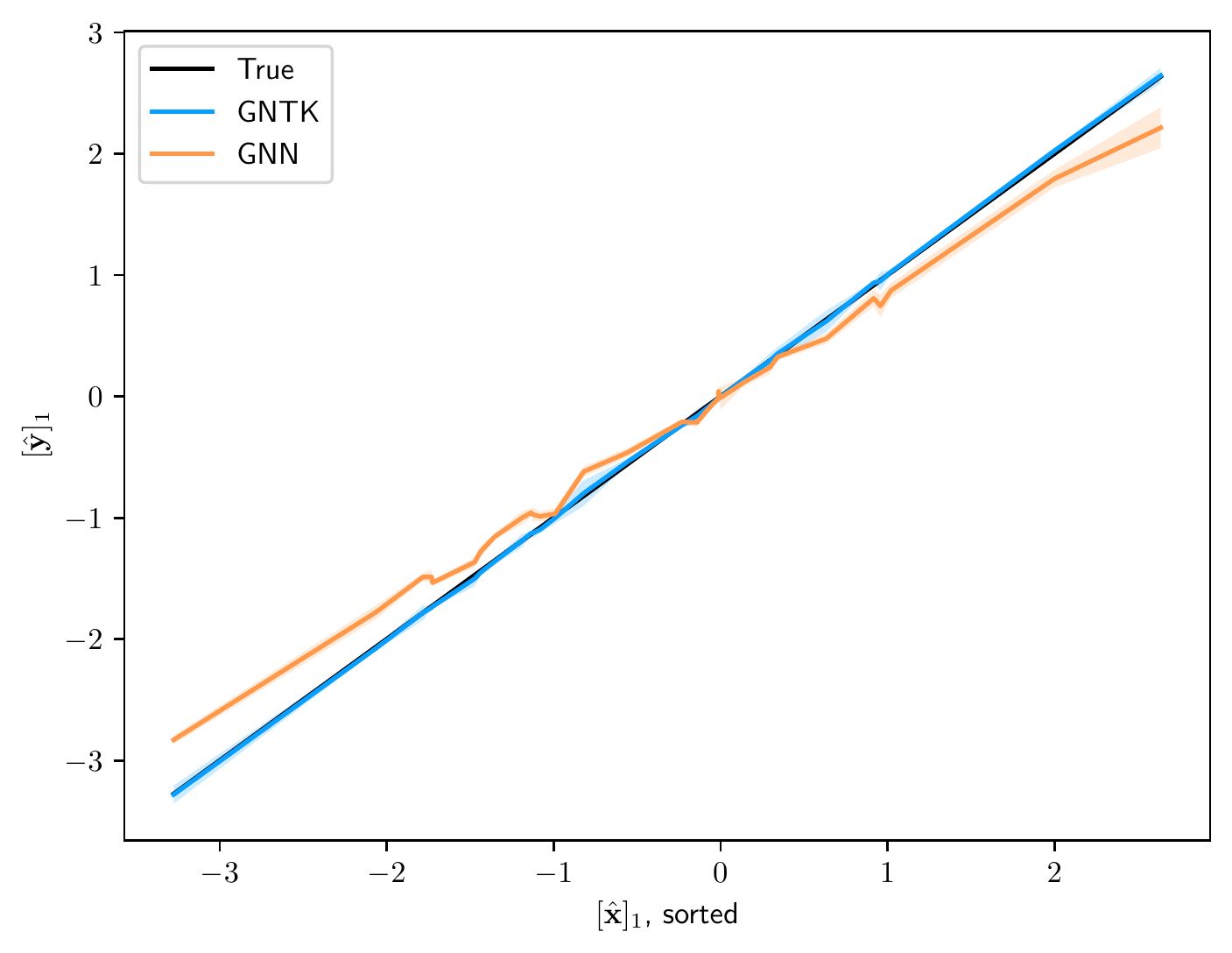}
\includegraphics[width=0.32\textwidth]{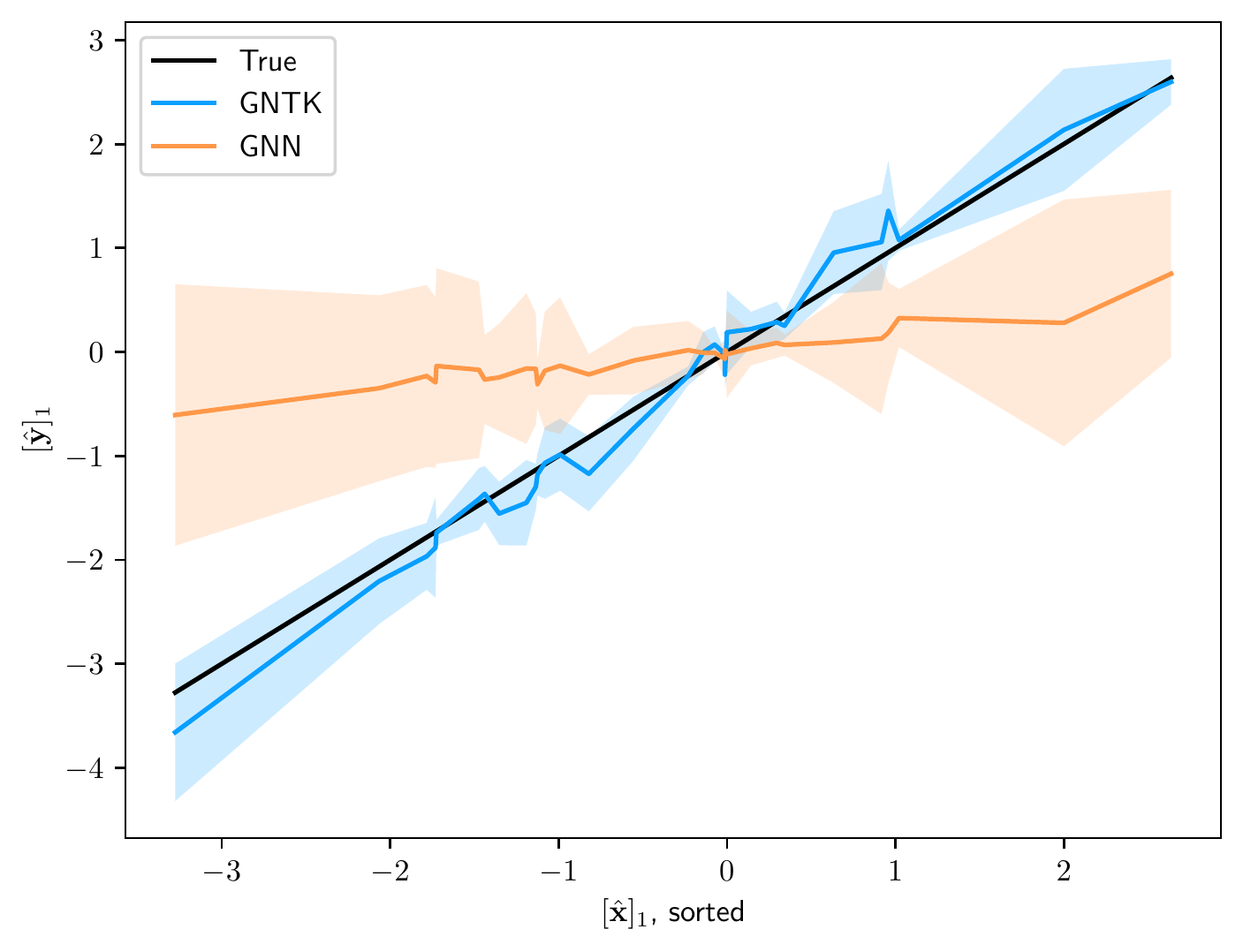}
\includegraphics[width=0.32\textwidth]{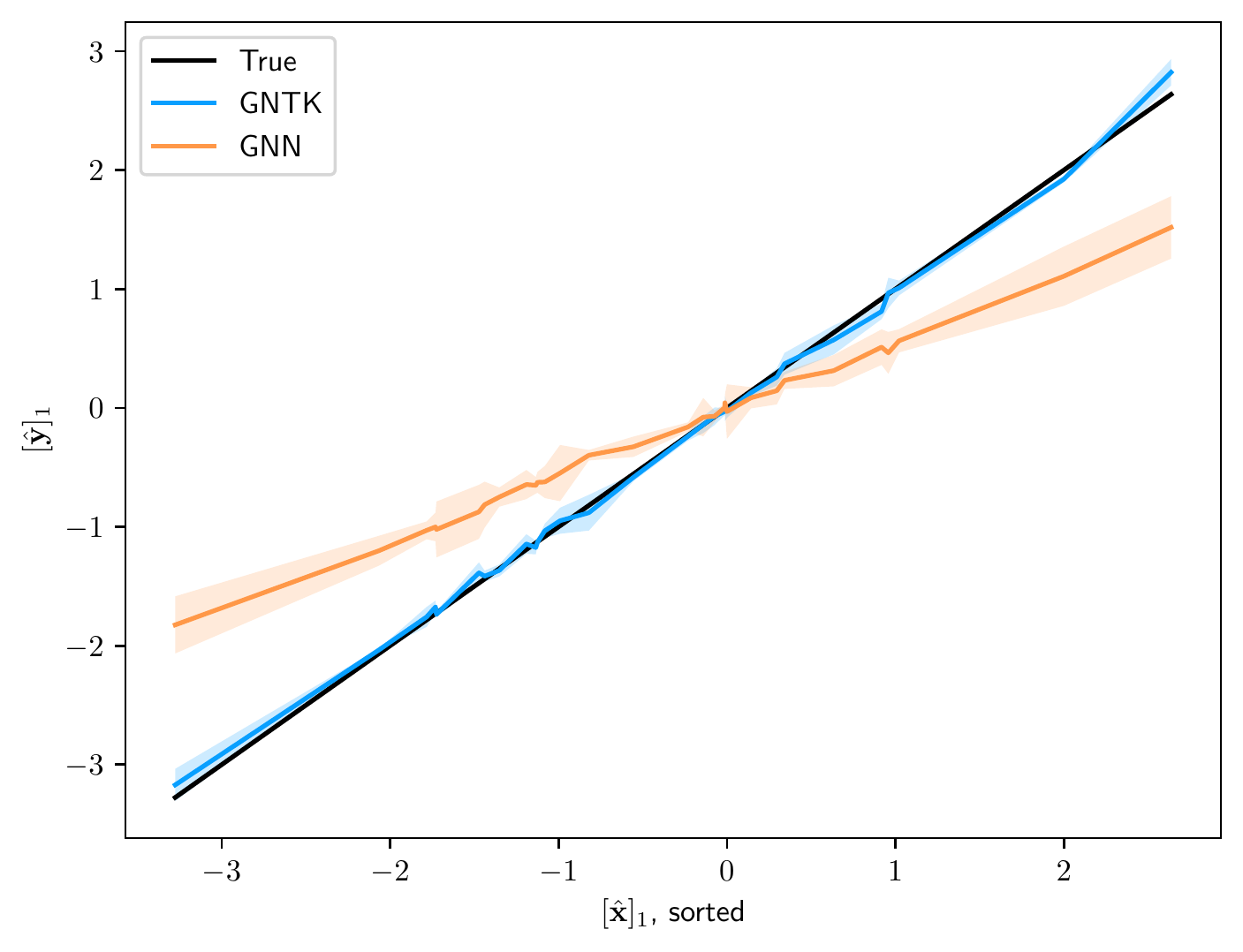}
\includegraphics[width=0.32\textwidth]{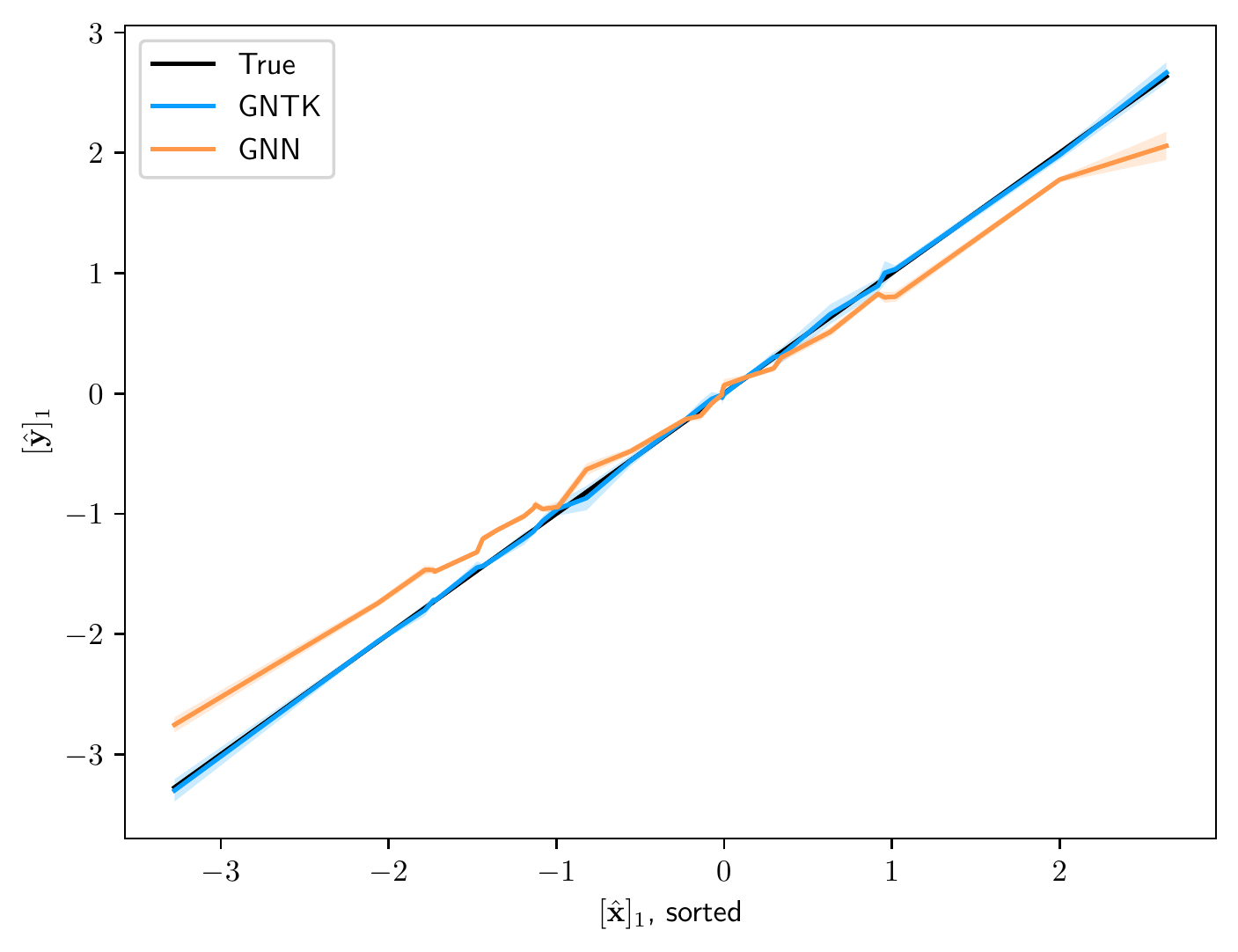}
\includegraphics[width=0.32\textwidth]{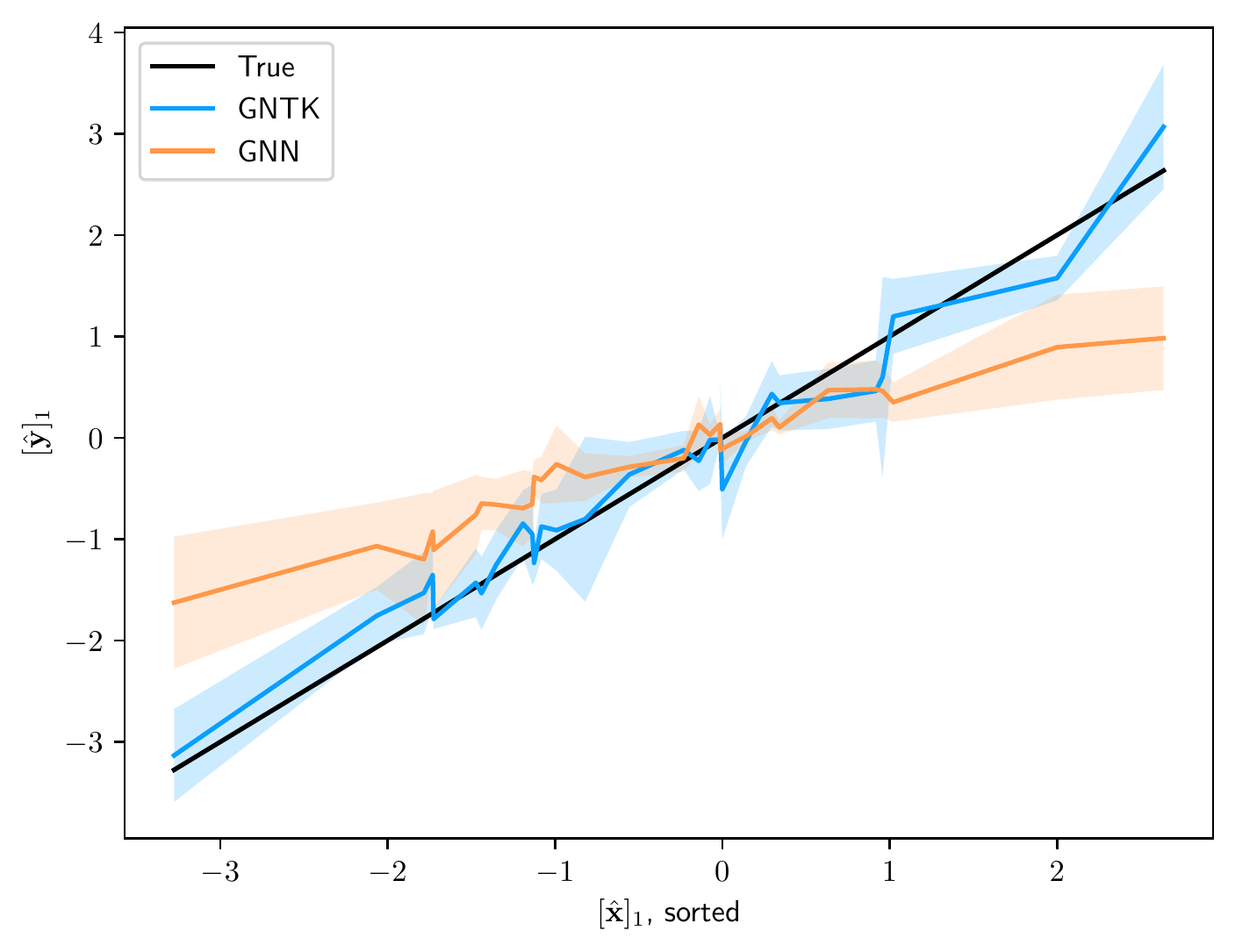}
\includegraphics[width=0.32\textwidth]{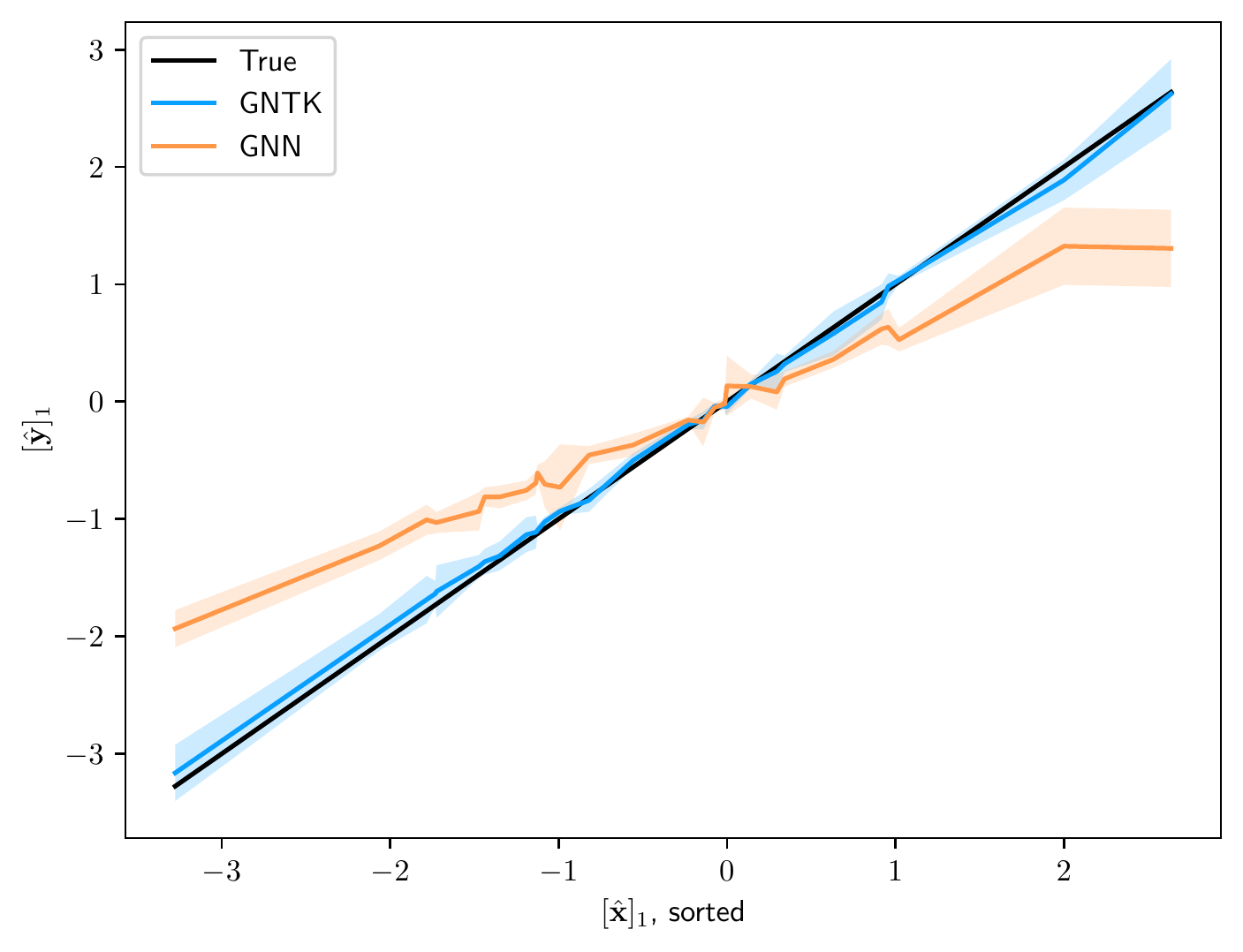}
\includegraphics[width=0.32\textwidth]{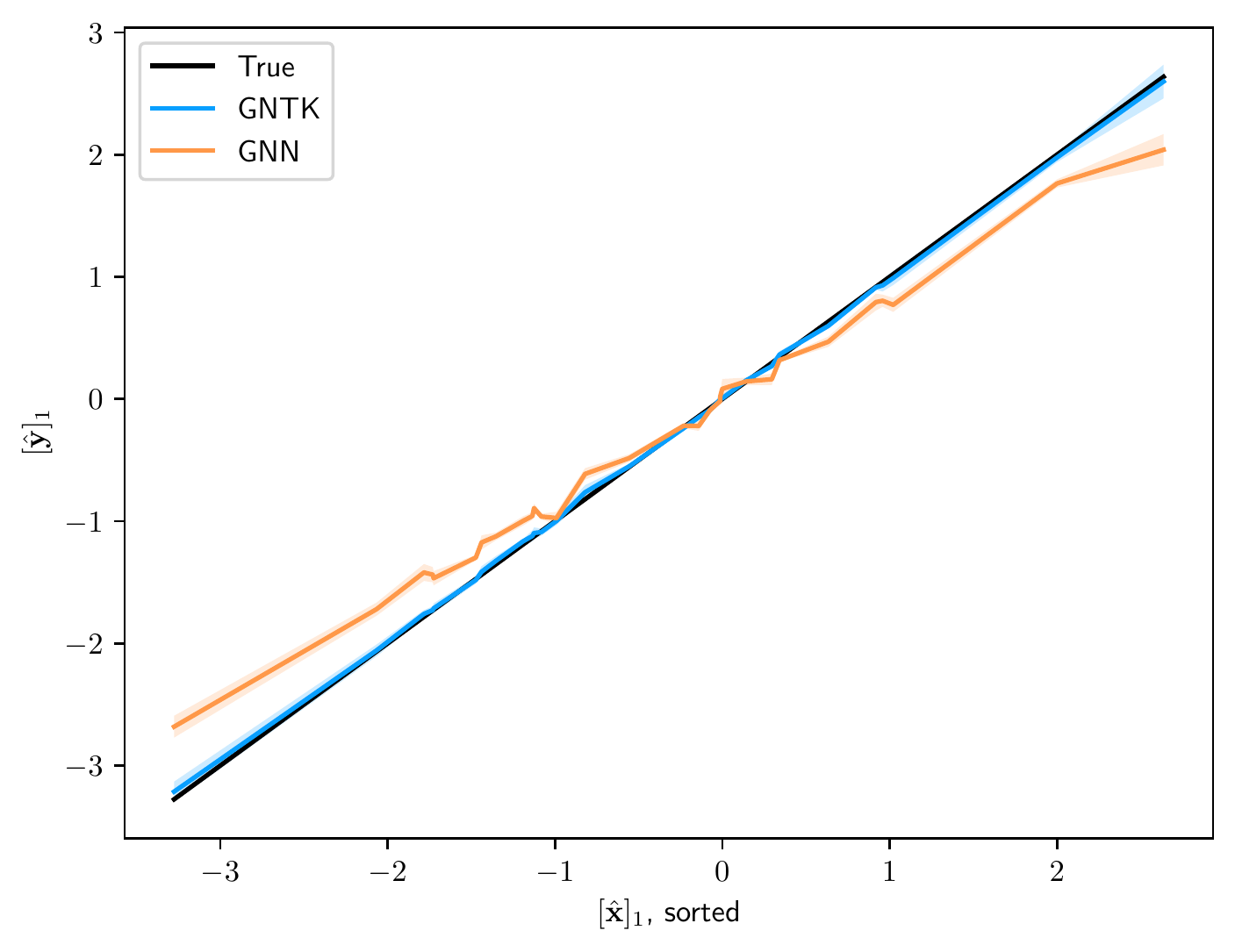}
\caption{Projections of the inputs and outputs of the GNTK and GNN, and of the target or true labels onto the first eigenvector of the adjacency matrix of a $n$-node SBM graph for widths $F=10$ (left), $F=50$ (center) and $F=250$ (right). From top to bottom $n=20,40,60,90,100$.}
\label{fig:appx}
\end{figure*}

\section{Wide Network Behavior: Additional Experiments}
\label{appendix:wide}

Under the same expeirmental setting of Sec. \ref{sbs:wide_sims}, in Fig. \ref{fig:appx}, we plot the projection of the outputs $\bby$ onto the first graph eigenvector $\bbv_1$, $[\hby]_1=\bbv_1^T\bby$ against the projections of the inputs onto the same vector $[\hbx]_1 = \bbv_1^T\bbx$ (sorted in ascending order) for both the GNNs and the GNTKs for $F^{(1)}_1=10$ (left), $F^{(2)}_1=50$ (center), and $F^{(3)}_1=250$ (right), and $n=20,40,60,80,100$ (top to bottom). The results are averaged over $5$ GNN initializations, with the solid lines representing the mean and the shaded areas representing the standard deviation. The behavior is as expected: as the width increases, the GNN and the GNTK have smaller variance in behavior for different weight initializations, and the GNN and GNTK curves align. As $n$ increases, we also observe a slight improvement in the variance of the GNN curves. This is expected as, when GNNs are trained on larger graphs, they have better transferability \cite{ruiz20-transf}. No significant trends are observed for the GNTK curves for increasing $n$.


\end{document}